\documentclass[twoside]{article}

\usepackage[preprint]{aistats2027}

\usepackage[round]{natbib}

\usepackage{xcolor}
\usepackage{wrapfig}
\usepackage{caption}
\usepackage{booktabs}
\usepackage{bm}
\usepackage{comment}
\usepackage{float}
\usepackage{graphicx}
\usepackage{amsfonts}
\usepackage{xspace}
\usepackage[ruled,vlined]{algorithm2e}
\usepackage{amssymb}
\usepackage{amsthm}
\usepackage{thmtools}
\usepackage{algorithmic}
\usepackage{hyperref}
\usepackage{cleveref}
\usepackage[T1]{fontenc}
\ifdefined\XeTeXversion
  \usepackage{newunicodechar}
  \newunicodechar{’}{'}
  \newunicodechar{‘}{`}
  \newunicodechar{–}{--}
\fi
\providecommand{\SO}{\mathrm{SO}}

\newcommand{\permn}{\mathfrak S_n}
\newcommand{\R}{\mathbb{R}}
\newcommand{\N}{\mathbb{N}}
\newcommand{\eb}{\mathbf{e}}
\newcommand{\xb}{\mathbf{x}}
\newcommand{\ab}{\mathbf{a}}
\newcommand{\ub}{\mathbf{u}}

\providecommand{\coloneqq}{\mathrel{:}=}
\newcommand{\our}{\textsc{LexiCan}\xspace}

\newtheorem{theorem}{Theorem}
\newtheorem{lemma}{Lemma}

\begin{document}
\runningtitle{LexiCan: Lexicographic Canonicalization of Point Clouds}
\runningauthor{Sanchez, Mihajlovic, Valeau, Beni\'c, Raj, Sarangi, Hordan}

\twocolumn[

\aistatstitle{\our: Lexicographic Canonicalization of Point Clouds for Machine Learning Applications}

\aistatsauthor{Daniel Sanchez\textsuperscript{*} \And Djordje Mihajlovic\textsuperscript{*} \And  Federica Valeau\textsuperscript{*}}

\aistatsaddress{Imperial College London \And  University of Edinburgh \And AMLab - University of Amsterdam} 

\aistatsauthor{Luka Beni\'c\textsuperscript{*} \And Raghuwansh Raj\textsuperscript{*} \And  Soutrik Sarangi\textsuperscript{*}}

\aistatsaddress{
Ru\dj{}er Bo\v{s}kovi\'c Institute
\\
University of Zagreb
\And
University of Luxembourg
\And
Microsoft Research India
}
\vspace{-4mm}
\aistatsauthor{Snir Hordan}

\aistatsaddress{Technion - Israel Institute of Technology}
]

\begingroup
\renewcommand{\thefootnote}{\fnsymbol{footnote}}
\footnotetext[1]{Equal contribution. These authors are listed alphabetically by first name.}
\endgroup

\begin{abstract}
Canonicalization resolves the rotational and permutation ambiguity of 3D point clouds by selecting a unique orbit representative, maintaining their geometric information while allowing standard architectures, such as Transformers, to operate on them. Existing canonicalizations, however, exhibit discontinuity jumps on molecular configurations and non-trivial invariance errors in popular molecular datasets, which lead to ambiguity in their representation. To bridge this gap, we introduce \our{}, a canonicalization that is provably invariant under either $\mathrm{O}(3)$ or $\SO(3)$, empirically invariant to rotations to within $\sim\!10^{-6}$ \AA{} on the QM9 and QM-sym molecular datasets, and mitigates discontinuity jumps on QM-sym, while running orders of magnitude faster than an existing canonicalization. Finally, non-invariant models perform better on QM9 property prediction and ModelNet40 shape classification with \our{} than with that canonicalization.
\end{abstract}

\section{Introduction}

\label{sec:intro}
Three-dimensional point clouds are an efficient and widely-used representation of geometric data in molecular modelling, computer vision, robotics, geometric processing, and scientific computing \citep{thomas2018tensor,gebauer2019symmetry,batzner2022e3,pomerleau2015review}. Unlike images expressed on a fixed pixel grid, point clouds generally do not come with a preferred coordinate system. The same object may be observed at arbitrary positions and orientations, whilst its geometric and physical properties remain unchanged. A reliable learning system must account for these transformations: invariant outputs, such as object classes or molecular energies, should not depend on the chosen reference frame, whereas geometric outputs, such as forces or trajectories, should transform accordingly. Incorporating these symmetries into the learning pipeline can substantially improve sample efficiency, robustness, and generalization by preventing a model from learning spurious dependencies only based on pose \citep{hinton2011transforming,cohen2017steerable,romero2020wavelet}.

Nevertheless, constructing a reliable canonical frame is difficult. Classical alignment techniques often rely on Principal Component Analysis (PCA), using the eigenvectors of the point cloud’s covariance matrix as coordinate axes \citep{vranic2001tools,chaouch2008novel,bellekens2014survey}. Yet, such frames are ambiguous under sign changes and become non-unique when covariance eigenvalues coincide. Frame Averaging methods retain multiple possible frames but do not fully resolve the degeneracies caused by repeated eigenvalues and can increase computational costs \citep{puny2022frame,duval2023faenet}. Learning-based canonicalization offers greater flexibility but introduces a different set of difficulties. A learned frame may depend on the symmetries represented in its training distribution, may generalize poorly to previously unseen stabilizer types, and can introduce errors into the canonicalized coordinates \citep{winter2022unsupervised}.

Exact canonicalizations of point clouds under the orthogonal $\mathrm{O}(3)$ and special orthogonal $\SO(3)$ groups exist. But, as recently shown, none of them can be made continuous \citep{dymlawrence2024impossibility, baker2024explicit}. To address this, \citet{baker2024explicit} introduce Asymmetric Unit Normalization (ASUN), an explicit construction that detects the point group of a point cloud to extract a component of the point clouds called an asymmetric unit and uses this component to align the point cloud to a canonical pose. In practice, however, we find that ASUN's implementation assigns different orientations to rotationally equivalent point clouds, beyond numerical error, and fails to canonicalize some symmetric configurations and is unstable on molecular geometries in the QM9 dataset \cite{ramakrishnan2014quantum}.

We address this theory-to-practice gap by introducing \textbf{\our}, a $3$D point cloud canonicalization based on lexicographic ordering, which is both invariant to within $\sim\!10^{-6}$ \AA{} on molecular data, handles symmetric point clouds and improved stability on molecular configurations.
Specifically, our main contributions are the following:

\begingroup
\setlength{\itemsep}{1pt}
\setlength{\parskip}{0pt}
\setlength{\topsep}{2pt}
\setlength{\partopsep}{0pt}
\begin{enumerate}
\item \textbf{Provable canonicalization.} We construct a provable canonicalization for labeled point clouds under groups $G_{3}^{\SO}=\permn\times\SO(3)$ and $G_{3}^{\mathrm{O}}=\permn\times\mathrm{O}(3)$, where $\permn$ is the symmetric group on $n$ elements, using circular words and least cyclic shifts, proving its correctness in Lemma~\ref{lem:circular-canonicalization}, Theorem~\ref{thm:three-dimensional-canonicalization}, and Theorem~\ref{thm:algorithm-correctness} (see Appendix~\ref{app:proofs}).

\item \textbf{Efficient implementation.} We provide an algorithm with a worst-case runtime complexity of $O(n^2\log n)$ (see Theorem~\ref{thm:algorithm-runtime} in Appendix~\ref{app:proofs}). However, empirically, across the evaluated input sizes, \our is faster than ASUN by a factor of $\sim\!\mathbf{130}$ to $\sim\!\mathbf{5{,}400}$ on ModelNet40 \citep{modelnet} and of $\sim\!\mathbf{6}$ to $\sim\!\mathbf{22}$ on QM9, growing with the number of points.

\item\textbf{Almost-everywhere local regularity.}
We prove that \our's exact construction and its tolerance-based
construction are locally Lipschitz
almost everywhere in the space of centered point clouds with fixed
labels (Theorem~\ref{thm:robustness}). This establishes generic local regularity and we assess stability on molecular configurations specifically
through our perturbation experiments.

\item \textbf{High-precision invariance.} On the QM9 dataset, \our achieves mean invariance errors of $\sim\!10^{-6}$ \AA{}, four to five orders of magnitude lower than those of PCA- and ASUN-based canonicalizations ($0.89$ and $0.04$ \AA{} on molecules of rank $3$) (see Section~\ref{sec:experiments}).

\item \textbf{Well-posedness at molecular geometries.} We show that while PCA and ASUN maintain a non-zero probability of an $\varepsilon$-perturbation inducing an $O(1)$ change in canonical form as $\varepsilon\to0$, for \our this probability is $0$ at $\varepsilon=10^{-8}$ (no jump in $64{,}000$ trials) on the QM-sym \citep{liang2019qmsym} dataset, and \our shows no jumps on any of the $82$ pilot geometries (see Table~\ref{tab:disc-methods}).

\end{enumerate}
\endgroup

\section{Related Work}
\label{sec:related}
\textbf{Point cloud canonicalization.} Classical methods construct representatives using deterministic geometric rules, such as PCA-based alignment, eigenvector canonicalization, or lexicographic ordering. Frame Averaging constructs a set of PCA-derived frames and averages predictions over them \citep{puny2022frame}, while ASUN \citep{baker2024explicit} detects point-group symmetries, extracts an asymmetric unit, and aligns it with a canonical frame. Yet, PCA has sign ambiguities of its eigenvectors (and eigenbasis ambiguities on self-symmetric point clouds), leading to a mapping that is not well-defined, while ASUN has non-trivial invariance errors on QM9 and discontinuity jumps on molecular structures, see \Cref{sec:experiments}. More related work is presented in Appendix~\ref{apdx:add-rlt-wrk}.

\section{Preliminaries}
\label{sec:setup}
We consider point clouds with additional point-wise labels as our data object for canonicalization. Formally, let $n\geq1$ be a positive integer, and let $\Lambda$ be a totally ordered label set (for example, atomic symbols in the case of molecules) and let $$P=\{(\xb_i,a_i)\}_{i=1}^n\in(\mathbb R^d\times\Lambda)^n$$ be a labeled point cloud. Here $\xb_i\in\mathbb R^d$ is a point and $a_i\in\Lambda$ its label.

Point clouds have natural symmetry to node reordering and rotations of the points, both proper and improper. These actions are formalized as the action of the symmetric group on $n$ elements, denoted by $\permn$, and the action of $\SO(d)$ and $\mathrm{O}(d)$ on each of the points, respectively. 

We denote the joint action of permutations and rotations by $G_{d}^{\SO}=\permn\times\SO(d)$ and $G_{d}^{\mathrm{O}}= \permn\times\mathrm{O}(d)$, where the first accounts for permutations and proper rotations and the second also includes reflections. The action of a group action $g$ on a point cloud $P$ is
\begin{equation}
\textstyle
(\pi,Q)\cdot P = \left(\bigl(Q\xb_{\pi^{-1}(i)},a_{\pi^{-1}(i)}\bigr)\right)_{i=1}^n
\label{eq:group-action}
\end{equation}

where $g\coloneqq(\pi,Q)\in G_{d}^{\kappa}$, where $d$ is the dimension and $\kappa$ stands for $\SO(d)$ or $\mathrm{O}(d)$.

The labels are unchanged by spatial transformations and are permuted together with their points. For point clouds having multiple attributes or labels, we consider the set of labels as $\mathcal A=\mathcal A_1\times\cdots\times\mathcal A_q$, where each $\mathcal A_r$ is totally ordered. This gives a natural lexicographic ordering of the entire label set.

\begin{figure*}[hbt!]
    \centering
    \includegraphics[width=\textwidth]{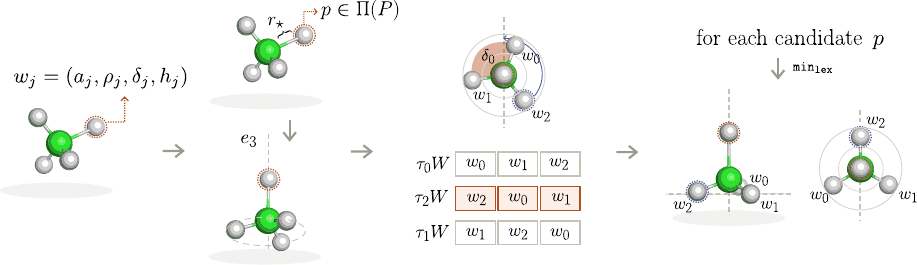}
    \caption{Visualization of \our algorithm. \textbf{(1)} Each point of the labeled and centered point cloud described by a string of values $w_{j}$. \textbf{(2)} A candidate set $\Pi(P)$ for the axis of projection is chosen based on radius, labels (here, one point at distance $r_{\star}$). The projection axis is then oriented with $e_{3}$. \textbf{(3)} 2D canonicalization is performed on the projected point cloud, where the lexicographically minimal starting point (here $w_{2}$) is rotated around axis and aligned with $+y$. \textbf{(4)} This is performed across all choices of $p\in \Pi(P)$ and the lexicographically minimal serialization yields the canonicalization.}
    \label{fig:canonicalization-algorithm}
\end{figure*}

Under this notation, a point cloud is written as:
\begin{equation}
\textstyle
    P=\bigl\{(\xb_i,\ab_i)\bigr\}_{i=1}^n,\,\ab_i\coloneqq(a_i^1,\ldots,a_i^q)\in\mathcal A~.
\label{eq:ptcloud_label}
\end{equation}
We center $P$ by $\xb_i\mapsto \xb_i - \mathbf{c}(P)$, where $\mathbf{c}(P)$ is the centroid of the point cloud, which is a translation canonicalization of the point cloud. We assume $P$ is centered henceforth, giving translational invariance, therefore we focus on canonicalizing point clouds with respect to the groups $G^{\SO}_d$ and $G^{\mathrm{O}}_d$.

We now formally define the term canonicalization for point clouds and use it throughout the text for the groups $G_{d}^{\SO}$ and $G_{d}^{\mathrm O}$, with $d=2$ and $d=3$.

\textbf{Definition 1 (Canonicalization for Point Clouds)}
\textit{A point cloud canonicalization for the group $G_{d}^{\kappa}$, where $d$ is the dimension and $\kappa$ stands for the groups $\SO$ or $\mathrm{O}$, is a function $\mathcal C_{d}^{\kappa}(\cdot) : (\mathbb R^d\times\Lambda)^n \to (\mathbb R^d\times\Lambda)^n$ on centered, labeled point clouds that satisfies, for every group element $g\in G_{d}^{\kappa}$,}
\begin{equation}
\textstyle
\label{def:canonicalization}
    \begin{aligned}
        & \hspace{1.4cm} \mathcal C_{d}^{\kappa}(g\cdot P)=\mathcal C_{d}^{\kappa}(P), \\
        & \mathcal C_{d}^{\kappa}(P)\in G_{d}^{\kappa}\cdot P \coloneqq \{g'\cdot P:g'\in G_{d}^{\kappa}\}.
    \end{aligned}
\end{equation}
That is, it is an invariant mapping that chooses a $G_{d}^{\kappa}$-orbit representative for each point cloud.

\section{\our: Lexicographic Canonicalization of Point Clouds}
\label{sec:canonicalization} 
In this section, we present \our and prove it satisfies the properties of a canonicalization of point clouds. We first build a canonicalization for 2D point clouds. To extend it to 3D, we select an invariant axis of rotation and then apply the 2D recipe to the projected planar cloud. \our handles node features and is also valid in their absence. An overview of the \our workflow is shown in Figure~\ref{fig:canonicalization-algorithm}.

\textbf{Labels and their ranks.} Each point carries one label $\ell_i\in\Lambda$ from a totally ordered set (the atomic label for molecules and unlabeled clouds use a single label). \our compares
labels through their \emph{rarest-first rank}: with $c(\lambda)=|\{i:\ell_i=\lambda\}|$
the number of points carrying label $\lambda$, the labels present in $P$ are sorted by
increasing $c(\lambda)$, ties broken by the order of $\Lambda$, and
$a_i\in\{0,1,\ldots\}$ is the position of $\ell_i$ in that list (the rarest label has
rank $0$). Ranks are unchanged by spatial transformations and permuted together with the points. In general, we denote node labels by $\ell_i$, but in the case they are vector valued, they are denoted by $\ab_i$.

\subsection{Canonicalization of a Labeled 2D Point Cloud}
\label{subsec:circular-canonicalization}
We draw inspiration from classical symmetry detection algorithms \citep{wolter1985optimal} to represent 2D point clouds as a discrete one-dimensional torus with rotation-invariant multi-dimensional continuous and discrete features, intuitively resembling a circular string. We then choose the lexicographically minimal \textit{starting point} and align it with the $y$ axis. Additionally, we extend this approach to handle reflections.

Any point cloud in 2D can be written as $P\in \mathbb{R}^{n \times 2}$, where each $P_i\in\R^2$ is the spatial coordinate of a point. Writing the position vector using polar coordinates and including the node labels (as in Eq.~\eqref{eq:ptcloud_label}), we can write a planar point cloud as: $P=\bigl\{(\rho_i,\phi_i,\ab_i)\bigr\}_{i=1}^n$, where $(\rho_i,\phi_i)$ are polar coordinates and $\ab_i$ are the label tuples. For the non-origin points, we create a canonical description through the following steps:

\textbf{(1)} We sort the points by increasing angle $\phi_i$. Points with the same angle are ordered by $(a_i,\rho_i)$. We denote the resulting sorted indices by $i_0,\ldots,i_{m-1}$. 

\textbf{(2)} For each point, let $\delta_j$ be the counterclockwise angular gap from $i_j$ to $i_{j+1}$, including the wrap-around from $i_{m-1}$ to $i_0$. 

\textbf{(3)}  We record each point by its radius, angular difference, and labels. This gives the circular word $W(P)=\left((a_{i_j},\rho_{i_j},\delta_j)\right)_{j=0}^{m-1}$.

\textbf{(4)} We read this word starting from every possible position and retain the lexicographically smallest reading: $W_{\min}(P)=\min_{0\leq s<m}\tau_s W(P)$, where $\tau_s$ denotes a cyclic shift by $s$ positions. The first point of $W_{\min}(P)$ therefore provides a canonical reference.

\noindent$\mathbf{SO(2)}$ \textbf{canonicalization.}
If all points lie at the origin, return them ordered by theiroriginal labels. Otherwise, rotate the cloud so that the first point of $W_{\min}(P)$ lies on the positive $Y$ axis.
If $\phi_\star$ is its angle, define
$\widehat\phi_i=\phi_i-\phi_\star+\pi/2\pmod{2\pi}$
for every non-origin point, and write $\widehat u_i=\rho_i(\cos\widehat\phi_i,\sin\widehat\phi_i)$ if $\rho_i > 0$ and $(0,0)$ otherwise. Let $\sigma$ order all point indices lexicographically by
$(a_i,\widehat u_{i,1},\widehat u_{i,2})$. Then:
\begin{equation}
\label{eq:circular-canonicalization-so}
\mathcal C_2^{\SO}(P)
=
\left(
\bigl(\widehat u_{\sigma(i)},\ab_{\sigma(i)}\bigr)
\right)_{i=1}^n.
\end{equation}
Points at the origin
are excluded from reference selection but included in the
final ordering.

\noindent $\mathbf{O(2)}$ \textbf{canonicalization.} To additionally remove reflections for the point cloud $P$, we do the following: 

\textbf{(1)} compute the $\SO(2)$ canonical form $\mathcal C_2^{\SO}(P)$; 

\textbf{(2)} reflect $P$ across the $X$ axis and compute the same canonical form;

\textbf{(3)} Call the reflected point cloud $\overline P=\bigl\{(\rho_i,-\phi_i\!\!\pmod{2\pi},\ab_i)\bigr\}_{i=1}^n.$ For $m>0$, let $s(P)$ be a starting index producing $W_{\min}(P)$
and write $j_t=i_{(s(P)+t)\bmod m}$.
After choosing this start, define
\begin{align*}
\textstyle
    B(P)\coloneqq
\left((a_{j_t})_{t=0}^{m-1},
(\delta_{(s(P)+t)\bmod m})_{t=0}^{m-1},
(\rho_{j_t})_{t=0}^{m-1}
\right).
\end{align*}

This comparison examines the complete label-rank sequence first,
then the complete angular-gap sequence, and finally the complete
radial sequence. Define $B(\overline P)$ by the same construction
applied to the reflected cloud. If $m=0$, both keys are empty. Define:
\begin{equation}
\textstyle
\label{eq:circular-canonicalization-o}
\mathcal C_2^{\mathrm O}(P)=
\begin{cases}
\mathcal C_2^{\SO}(P),
& B(P)\leq_{\mathrm{lex}}B(\overline P),\\
\mathcal C_2^{\SO}(\overline P),
& B(\overline P)<_{\mathrm{lex}}B(P).
\end{cases}
\end{equation}
If the two keys agree, the two branches produce the same
exact labeled canonical form. For projected 3D clouds, append the complete height sequence
$(h_{j_t})_{t=0}^{m-1}$ as the final component of $B(P)$.

\subsection{Canonicalization of a Labeled 3D Point Cloud}
We extend the 2D canonicalization to a 3D one via the insight that once we fix an anchor axis in 3D, the projections of the points onto its orthogonal hyperplane form a 2D point cloud. This reduces the problem of 3D canonicalization to a 2D one, which we have solved in Section~\ref{subsec:circular-canonicalization}. Using lexicographic sorting and only invariant features, we can choose this axis invariantly to permutations, rotations, and translations, thus allowing us to obtain a 3D canonicalization.

Let $P$ be a centered labeled point cloud in $\R^3$. We reduce its canonicalization to the 2D construction above via the following steps:

\textbf{(1)} We choose a candidate set for the axis based on invariant features: the largest radius and, among those points, the rarest label as follows:
\begin{equation}
\textstyle
\label{eq:candidate-poles}
\begin{split}
    r_\star&=\max_i\|\xb_i\|,\\
    \Pi(P)&=\left\{p:\|\xb_p\|=r_\star,\ a_p=\min_{\|\xb_i\|=r_\star}a_i\right\}.
\end{split}
\end{equation}
\textbf{(2)} For every $p\in\Pi(P)$, choose a rotation $A_p$ that places $\xb_p$ on the positive $Z$ axis. We can write $A_p\xb_i=(\ub_i^{(p)},h_i^{(p)})$, where $\ub_i^{(p)}\in\R^2$ is the planar projection and $h_i^{(p)}$ is the height.  The planar cloud keeps the height as an extra label:
$P_p=\{(\ub_i^{(p)},\ab_i,h_i^{(p)})\}_{i=1}^n$, and its circular word is
$W(P_p)=\{(\ab_{i_j},\rho_{i_j},\delta_j,h_{i_j})\}_{j=0}^{m-1}$, with points of equal angle
ordered by $(\ab_i,\rho_i,h_i)$; that is, the label rank first and the height last.

\textbf{(3)} Once a candidate axis is fixed, we canonicalize the projected planar cloud using the 2D canonicalizations defined in the preceding subsection. For every possible candidate $p\in\Pi(P)$, we perform this \textit{project and canonicalize} procedure and select the output with the minimum lexicographic order.

\noindent$\mathbf{SO(3)}$ \textbf{canonicalization.}
For each $p\in\Pi(P)$, apply $\mathcal C_2^{\SO}$ to $P_p$,
including height and retaining the original labels. Restore each height $h$ as the third coordinate:
$(\widehat u,\ab,h)\mapsto
((\widehat u_1,\widehat u_2,h),\ab)$.
Order the resulting labeled 3D points lexicographically by
$(a,x,y,z)$, where $a$ is the  rarity rank,
and denote this ordered candidate by $Y_p^{\SO}$. We define the canonicalization as:
\begin{equation}
\textstyle
\label{eq:canonicalization-so3}
\mathcal C_3^{\SO}(P) = \min_{p\in\Pi(P)}Y_p^{\SO},
\end{equation}
\noindent$\mathbf{O(3)}$ \textbf{canonicalization.} We perform the same construction as before $\forall p\in\Pi(P)$ using $\mathcal C_2^{\mathrm O}(\cdot)$ instead of $\mathcal C_2^{\SO}(\cdot)$ and call the resulting 3D cloud $Y_p^{\mathrm O}$. We define the canonicalization as:
\begin{equation}
\textstyle
\label{eq:canonicalization-o3}
\mathcal C_3^{\mathrm O}(P)=\min_{p\in\Pi(P)}Y_p^{\mathrm O}.
\end{equation}
By comparing the reflected planar projection, we remove 3D reflections in this case. In Theorem~\ref{thm:three-dimensional-canonicalization}, we prove that $\mathcal C_3^{\SO}(\cdot)$ and $\mathcal C_3^{\mathrm O}(\cdot)$ are canonical representations for $G_3^{\SO}$ and $G_3^{\mathrm{O}}$, respectively. The theorem statement and formal proofs are in Appendix~\ref{app:proofs}.

\subsection{Algorithms for Computing the Canonicalizations}
\label{subse:algorithms}
\begin{algorithm}
\caption{\our}
\label{alg:canonicalization-overview}
\KwIn{Centered labeled cloud $P$}
\If{all points are at the origin}{
    \KwRet{points ordered by their labels}
}
Construct the admissible pole set $\Pi(P)$
using Eq.~\eqref{eq:candidate-poles}\\
\For{each $p\in\Pi(P)$}{
    Align $p$ with $\eb_3$ and construct $P_p$
    as in Algorithm~\ref{alg:three-dimensional-canonicalization}\\
    Apply the planar canonicalization to $P_p$,
    using Algorithm~\ref{alg:linear-cyclic-shift}
    for the least cyclic shift\\
    Restore heights and order the lifted records by
    $(a,x,y,z)$ to obtain $Y_p^\kappa$\\
    \If{this is the first candidate}{
        $D_{\mathrm{best}} \gets D_p$
    }
    \ElseIf{$D_p <_{\mathrm{lex}} D_{\mathrm{best}}$}{
        $D_{\mathrm{best}} \gets D_p$
    }
}
\KwRet{$D_{\mathrm{best}}$}
\end{algorithm}
The \our algorithm (see Algorithm~\ref{alg:canonicalization-overview}), that computes the canonicalizations $\mathcal C_3^{\SO}(\cdot)$ and $\mathcal C_3^{\mathrm O}(\cdot)$, is divided into two subroutines: Algorithm~\ref{alg:linear-cyclic-shift} for the canonicalization of a 2D point cloud and Algorithm~\ref{alg:three-dimensional-canonicalization} that computes the required 3D canonicalization by repeatedly using Algorithm~\ref{alg:linear-cyclic-shift} after selecting an axis of projection. 

Specifically, we first select an axis of projection based on frame invariant features like radius and label and project the 3D point cloud into a 2D one from the chosen axis. On this 2D projection, we apply Algorithm~\ref{alg:linear-cyclic-shift}, which is a two-pointer based approach inspired by Booth's algorithm \citep{Booth1980} to find the lexicographically minimum representation of a circular string. In case of ties in choosing the axis of projection, we perform this projection followed by Algorithm~\ref{alg:linear-cyclic-shift} for all possible axis choices and then select the one that produces the lexicographically smallest output. More details on the algorithm are given in Appendix~\ref{app:algorithms}.

In Theorem~\ref{thm:algorithm-correctness}, we show that \our computes the canonicalizations $\mathcal C_3^\kappa(\cdot)$, where $\kappa$ is $\SO$ or $\mathrm{O}$, and analyze the complexity in Theorem~\ref{thm:algorithm-runtime} (Appendix~\ref{app:proofs}). \our has a worst case runtime of $O(n^2\log n)$ since the number of ties $k$ for the axis choice is $O(n)$ in the worst case. But in most practical settings, $k$ grows sublinear in $n$. For datasets such as QM9 or ModelNet40, we have $k\approx O(1)$ and the runtime is log-linear in $n$, as shown in Table~\ref{tab:runtime-real-data}.

\subsection{Robustness of the \our Algorithm}
The canonicalizations $\mathcal C_3^\kappa(\cdot)$ with $\kappa\in\{\SO,\mathrm{O}\}$, are defined using exact comparison, but for their numerical implementations \our uses a fixed relative tolerance $\eta>0$: two lengths differing by at most $\eta R$, where $R$ is the largest distance of a point from the centroid, and two angles differing by at most $\eta$ radians may be treated as tied. We write $\widehat{\mathcal C}_{3,\eta}^{\kappa}(\cdot)$ for the implemented numerical map with tolerance $\eta$.

It is known that the $\mathcal C_3^{\kappa}(\cdot)$ canonicalizations cannot be globally continuous \citep{dymlawrence2024impossibility}. But we show that for all $\eta\geq 0$, the map $\widehat{\mathcal C}_{3,\eta}^{\kappa}(\cdot)$ is locally Lipschitz almost everywhere, establishing the robustness of the exact maps $\mathcal C_{3}^\kappa(\cdot)$ as well as that of the numerical approximations $\widehat{\mathcal C}_{3,\eta}^{\kappa}(\cdot)$.

In the following Theorem~\ref{thm:robustness} (proof in Appendix~\ref{app:proofs}) we show that for almost every point cloud $P_0$, the canonical map $\widehat{\mathcal C}_{3,\eta}^{\kappa}(\cdot)$ is Lipschitz in a neighborhood of $P_0$, thus proving almost everywhere continuity of the implemented map accounting for numerical approximations. This is due to the fact that sorting is piecewise continuous, although discontinuous when tie-breaking occurs, and the points of discontinuity form a Lebesgue-measure zero set, which then implies that \our satisfies this as well.  It should be noted that the analysis does not account for floating-point roundoff or machine-range restrictions.

\begin{restatable}[\textbf{Locally Lipschitz almost everywhere property of \our}]{theorem}{thmrobustness}
    \label{thm:robustness}
    Fix the number of points, their labels, the group $G_{3}^{\kappa}$ with $\kappa\in{\{\SO,\mathrm O}\}$, and a value $\eta\geq0$. 
    
    For almost every centered point cloud $P_0$, there exist $\delta_\eta(P_0)>0$ and $L_\eta(P_0)<\infty$ such that $\forall\,P, P'\in\mathcal{B}_\infty(P_0, \delta_\eta(P_0))$ it holds that $$\mathrm{EMD}\left(\widehat{\mathcal C}_{3,\eta}^\kappa(P), \widehat{\mathcal C}_{3,\eta}^\kappa(P')\right) \leq L_\eta(P_0)d_\infty\left(P, P'\right)$$ Consequently, for every fixed $\eta\geq0$, the map $\widehat{\mathcal C}_{3,\eta}^{\kappa}(\cdot)$ is continuous almost everywhere.
\end{restatable}

The local Lipschitz property guarantees that small perturbations to a generic point cloud $P$ do not cause its canonical representation $\mathcal C^\kappa_3(P)$, as produced by \our, to change discontinuously. Such a robustness result is not proven in the literature for existing canonicalization methods like ASUN, and we empirically validate this in Tables~\ref{tab:disc-methods} and \ref{tab:path-all}, where \our is more robust to small perturbations compared to ASUN and PCA.

\textbf{Dependence of \our on numerical tolerance.} The numerical tolerance $\eta$ plays an important role in \our. Since \our is a sorting-based algorithm, the quantities considered tied below a certain numerical threshold play an important role in the properties discussed below. 

\emph{Robustness:} As $\eta\to0$, the algorithm becomes more sensitive to small perturbations because most of them now start to cross the tolerance threshold, \emph{Canonicalization invariance:} As $\eta$ increases, the exactness of the sorting ceases to exist since all differences below $\eta$ are now considered tied. Consequently, the numerical implementation starts to differ from Algorithms~\ref{alg:canonicalization-overview} and \ref{alg:three-dimensional-canonicalization}. Thus different poses may start to output different canonicalizations, and \emph{Runtime:} As $\eta$ increases, the number of tied points increases (the cardinality of $\Pi(P)$ from Algorithm~\ref{alg:canonicalization-overview} increases). So, more tie-breaking comparisons need to be made, increasing runtime. 

Tables~\ref{tab:tol-sweep-zoo} and \ref{tab:tol-sweep-qmsym} from Appendix~\ref{app:tolerance-ASUN-PCA} give a detailed analysis on how robustness and runtime vary according with respect to tolerance, and we have subsequently chosen the tolerances of \our (as well as of PCA and ASUN) to achieve the optimal trade-offs.

\section{Experiments}
\label{sec:experiments}
In this section, we test the validity, efficiency, robustness, and end-to-end learning performance of \our against ASUN \citep{baker2024explicit}, and centered PCA, which are leading methods for point cloud canonicalization. In all the experiments of this section, the numerical tolerances of the corresponding canonicalization algorithms are taken as: $10^{-2}$ for ASUN (absolute, in \AA) and $10^{-3}$ for \our (relative to the cloud's radius); PCA has no tolerance. This choice was based on optimal performance for each canonicalization, see Appendix~\ref{app:just-tolerance} for more details. We discuss more experiments and setup details in Appendix~\ref{app:expt}.

Our data and code are publicly available at \url{https://anonymous.4open.science/r/lexican-anon/}.

\subsection{Runtime Analysis}
\label{subsec:runtime}
We consider the wall-clock runtime of \our against ASUN and PCA on the ModelNet40 and QM9 datasets. We report the wall-clock runtime in milliseconds (ms) grouped by the number of points in the point cloud in Table~\ref{tab:runtime-real-data}. As observed, \our is faster than ASUN at every size, by $130$ to $5{,}400\times$ on ModelNet40 and by $6$ to $22\times$ on QM9. PCA, which computes a single eigendecomposition, is the fastest method, about $2$ to $3\times$ faster than \our, but it is not a proper canonicalization: each eigenvector is defined only up to sign (and up to an orthogonal transformation within a repeated eigenvalue), so PCA is not invariant (Tables~\ref{tab:rankwise} and \ref{tab:qmsym}).

\begin{table}[H]
\vspace{-12.5pt}
\small
\setlength{\tabcolsep}{3pt}
\caption{Per-cloud canonicalization runtime in milliseconds (ms) on ModelNet40 and QM9, mean $\pm$ standard error over clouds of the per-cloud median of $10$ calls, one thread, for $\SO(3)$ canonicalization. Best results are in bold.}
\label{tab:runtime-real-data}
\resizebox{\columnwidth}{!}{
\begin{tabular}{@{}lrrr@{}}
\toprule
\multicolumn{4}{c}{\textbf{ModelNet40}} \\
\midrule
\#Points & \our & PCA & ASUN \\
\midrule
$128$ & $0.428{\pm}0.009$ & $\mathbf{0.189{\pm}0.001}$ & $56.8{\pm}0.2$ \\
$256$ & $0.467{\pm}0.008$ & $\mathbf{0.204{\pm}0.001}$ & $158.6{\pm}0.5$ \\
$512$ & $0.574{\pm}0.012$ & $\mathbf{0.231{\pm}0.001}$ & $562.7{\pm}1.4$ \\
$1{,}024$ & $0.782{\pm}0.023$ & $\mathbf{0.279{\pm}0.001}$ & $1{,}873.0{\pm}1.8$ \\
$2{,}048$ & $1.274{\pm}0.042$ & $\mathbf{0.374{\pm}0.001}$ & $6{,}883.9{\pm}12.4$ \\
\midrule
\multicolumn{4}{c}{\textbf{QM9}} \\
\midrule
\#Atoms & \our & PCA & ASUN \\
\midrule
$6$--$10$ & $0.411{\pm}0.051$ & $\mathbf{0.171{\pm}0.004}$ & $2.490{\pm}0.113$ \\
$11$--$15$ & $0.358{\pm}0.003$ & $\mathbf{0.169{\pm}0.000}$ & $3.940{\pm}0.037$ \\
$16$--$20$ & $0.360{\pm}0.001$ & $\mathbf{0.171{\pm}0.000}$ & $5.186{\pm}0.016$ \\
$21$--$25$ & $0.363{\pm}0.003$ & $\mathbf{0.170{\pm}0.000}$ & $6.302{\pm}0.024$ \\
$26$--$29$ & $0.355{\pm}0.006$ & $\mathbf{0.169{\pm}0.005}$ & $7.872{\pm}0.056$ \\
\bottomrule
\end{tabular}
}
\vspace{-12.5pt}
\end{table}

\subsection{Rank-wise Invariance Error}
\label{subsec:rank_wise}
We measure consistency by comparing canonical forms before and
after random rotations and translations. We report the mean Earth Mover's Distance (EMD) in \AA{}, lower is better.

\begin{table}[H]
\centering
\setlength{\tabcolsep}{2pt}
\caption{Invariance error by rank: mean EMD (\AA, two significant digits) between the canonical form of each QM9 molecule and that of a random rotation, translation and permutation of it. Best results are in bold.}
\label{tab:rankwise}
\begin{tabular}{@{}l@{\hspace{10pt}}rrr@{}} 
\toprule
Method & \multicolumn{1}{r}{R1} & \multicolumn{1}{r}{R2} & \multicolumn{1}{r}{R3} \\
\midrule
PCA & $\mathbf{3.5{\times}10^{-16}}$ & $5.4{\times}10^{-16}$ & $8.9{\times}10^{-1}$ \\
ASUN & $1.3{\times}10^{-15}$ & $1.3{\times}10^{-15}$ & $3.8{\times}10^{-2}$ \\
\our & $\mathbf{3.5{\times}10^{-16}}$ & $\mathbf{4.2{\times}10^{-16}}$ & $\mathbf{1.8{\times}10^{-6}}$ \\
\bottomrule
\end{tabular}%
\end{table}

From Table~\ref{tab:rankwise}, we make the following observations: \textbf{(1)} \our stays invariant at every rank, with a mean error of at most $2{\times}10^{-6}$ \AA{}, \textbf{(2)} for the molecules of rank $3$, neither ASUN nor PCA stays invariant.

\subsection{Invariance Error on QM-sym}
\label{subsec:qmsym}
In Table~\ref{tab:qmsym} we test the reliability of \our versus PCA and ASUN baselines on the QM-sym dataset, containing molecules with non-trivial self-symmetries. We rotate and translate each molecule ten times at random and measure the EMD between the canonical form of the molecule and that of each copy. We classify failure as a copy that surpasses the $0.1$ \AA{} EMD threshold, which is the error in measurement range of molecules in the QM9 dataset \citep{kim2019energy}. 

\begin{table}[H]
\centering
\caption{Symmetry group canonicalization errors: mean EMD (in \AA) over random transformations in QM-sym. Best results are in bold.}
\label{tab:qmsym}
\resizebox{\columnwidth}{!}{%
\begin{tabular}{@{}l r cc cc cc@{}}
\toprule
& & \multicolumn{2}{c}{PCA} & \multicolumn{2}{c}{ASUN} & \multicolumn{2}{c}{\our} \\
\cmidrule(lr{2pt}){3-4} \cmidrule(lr{2pt}){5-6} \cmidrule(lr{2pt}){7-8}
Group & $N_{\mathrm{mol}}$ & Mean & Fail\% & Mean& Fail\% & Mean & Fail\% \\
\midrule
$C_{2h}$ & $81{,}564$ & $0.564$ & $99.9$ & $0.616$ & $50.3$ & $\mathbf{0.000}$ & $\mathbf{0.0}$ \\
$C_{3h}$ & $81{,}017$ & $1.715$ & $100.0$ & $0.155$ & $32.0$ & $\mathbf{0.000}$ & $\mathbf{0.0}$ \\
$C_{4h}$ & $25{,}668$ & $1.447$ & $100.0$ & $0.267$ & $60.0$ & $\mathbf{0.000}$ & $\mathbf{0.0}$ \\
\midrule
All & $188{,}249$ & $1.180$ & $99.9$ & $0.370$ & $43.8$ & $\mathbf{0.000}$ & $\mathbf{0.0}$ \\
\bottomrule
\end{tabular}%
}
\end{table}

As observed, \textbf{(1)} \our stays within the $0.1$ \AA{} threshold on every molecule in QM-sym (mean EMD $2{\times}10^{-7}$ \AA{}), while the baselines do not, \textbf{(2)} ASUN breaks its own invariance on $43.8\%$ of molecules and PCA breaks on $99.9\%$.

\subsection{Perturbation Stability on Molecular Configurations}
\label{subsec:continuous}

Almost-everywhere regularity in the ambient point-cloud space
does not determine behavior on molecular geometries, which may
belong to lower-dimensional families. We therefore evaluate
perturbation sensitivity directly on QM-sym. For each tested
perturbation magnitude $\varepsilon$, we report the empirical
fraction $\widehat{p}(\varepsilon)$ of trials in which the EMD
between the canonical forms of the original and perturbed
clouds exceeds the jump threshold of $0.1$. These measurements
characterize finite-scale sensitivity.

In Table~\ref{tab:disc-methods} we see that \textbf{(1)} PCA and ASUN converge to $0.171$ and $0.031$, respectively, while \our reaches $0$ (no jump in $64{,}000$ trials at $\varepsilon=10^{-8}$), \textbf{(2)} \our shows no jumps on the pilot set, while PCA and ASUN jump on $33/82$ and $9/75$ geometries.

\begin{table}[H]
\centering
\small
\caption{The $\varepsilon\to0$ limit on QM-sym geometries. $\alpha_{99}$ is the $99$-th percentile amplification, pooled over all $\varepsilon$, while Jumps specifies the number of geometries that show a jump on the pilot set. Best results are in bold.}
\label{tab:disc-methods}
\begin{tabular}{@{}cccc@{}}
\toprule $p(\varepsilon)$ & PCA & ASUN & \our \\
\midrule
$p(10^{-8})$ & $0.171$ & $0.031$ & $\mathbf{0.000}$ \\
$p(10^{-6})$ & $0.515$ & $0.032$ & $\mathbf{0.001}$ \\
$p(10^{-4})$ & $0.519$ & $0.094$ & $\mathbf{0.093}$ \\
$p(10^{-2})$ & $0.543$ & $0.786$ & $\mathbf{0.522}$ \\
\midrule $\alpha_{99}$ & $2.7{\times}10^{7}$ & $7.5{\times}10^{5}$ & $\mathbf{3{\times}10^{3}}$ \\
Jumps & $33/82$ & $9/75$ & $\mathbf{0/82}$ \\
\bottomrule
\end{tabular}%
\end{table}

We also performed the analysis on a pilot set of $82$ example point clouds, comprised of QM9 and GDB structures in their equilibrium geometries, together with synthetic high-symmetry solids (specified in Appendix~\ref{app:symmetric-polyhedra}). Resolving  $p(\varepsilon)$ per geometry (see Figure~\ref{fig:geometry-census} in Appendix~\ref{app:additional_results}) shows that failures are widespread rather than concentrated on a few specific inputs: PCA exhibits discontinuities on $33$ of $82$ geometries, and ASUN on $9$ of the $75$ it can process, at a median of $82\%$ and $36\%$ across all discontinuous directions respectively, while \our is continuous across all geometries.

\subsection{Stability Under Controlled Perturbation} 
\label{subsec:stability}
To further test robustness to perturbation of \our, in Table~\ref{tab:path-all} we consider a set of four common molecules (NH$_{3}$, CH$_{4}$, H$_{2}$O$_{2}$, and C$_6$H$_6$). We define a one-parameter path $\gamma : [0,1] \longrightarrow (\mathbb{R}^3 \times \Lambda)^n$, with $\gamma(t) = P_t = \big((\mathbf{x}_i(t), \ell_i)\big)_{i=1}^n$, from each symmetric molecule $P_0$ in its equilibrium geometry, to a fully perturbed geometry $P_1$. We measure discontinuity, which we define as a difference in canonical form by more than $0.1$ in the radius of the point cloud (given by $r=\sqrt{\frac{1}{n}\sum_{i} \lVert \mathbf{x}_{i} \rVert^{2}}$) along $\gamma(t)$, between the canonicalized output at $t_{i}$ and $t_{i+1}$ along $600$ steps. 

\begin{table}[H]
\centering
\caption{Number of discontinuities (summed over $12$ rotated copies) along controlled perturbation paths, with the largest single move (in \AA{}) in parentheses. Best results are in bold.}
\label{tab:path-all}
\resizebox{\columnwidth}{!}{%
\begin{tabular}{@{}l rrr@{}}
\toprule
Path & PCA & ASUN & \our \\
\midrule
$\mathrm{NH}_3$ stretch & $17$~($0.875$) & $\mathbf{0}$~($0.001$) & $\mathbf{0}$~($0.000$) \\
$\mathrm{NH}_3$ bend & $17$~($0.860$) & $\mathbf{0}$~($0.011$) & $\mathbf{0}$~($0.001$) \\
$\mathrm{CH}_4$ stretch & $5301$~($1.019$) & $12$~($0.802$) & $\mathbf{0}$~($0.001$) \\
$\mathrm{C}_6\mathrm{H}_6$ stretch & $4$~($0.742$) & $\mathbf{0}$~($0.018$) & $\mathbf{0}$~($0.000$) \\
$\mathrm{H}_2\mathrm{O}_2$ torsion & $\mathbf{3}$~($0.745$) & $46$~($0.804$) & $24$~($0.428$) \\
\midrule
Max move & $1.019$ & $0.804$ & $\mathbf{0.428}$ \\
\bottomrule
\end{tabular}%
}
\end{table}

As observed,  \our has no discontinuity on $4/5$ paths and never more discontinuities than ASUN; on the $\mathrm{H}_2\mathrm{O}_2$ torsion, which passes through a symmetric configuration, PCA has fewer ($3$ against $24$) and \our has the lowest overall maximum move. We extend these tests to explore more general perturbation effects on a larger set of 18 representative clouds in Appendix~\ref{app:symmetric-polyhedra}.

We further test the robustness of \our and baselines in a real dataset of symmetric molecules, namely QM-sym. We quantify our results by considering the ``flip rate'', the amplification factor and the failure to be invariant. In Table~\ref{tab:PCA-ASUN-Lexi-flip-qmsym},  flip rate is the fraction of perturbations whose canonical form varied $10\times$ further than the input moved, $\alpha$ is the output movement divided by the input movement (1 means no amplification). Both of the values, if lower, imply robustness. 

\begin{table}[H]
\centering
\caption{Robustness in QM-sym ($1{,}000$ molecules). Not invariant counts the molecules with one of $12$ random views more than $0.1$~\AA{} (EMD) from the canonical form, out of those the algorithm canonicalizes. Best results are in bold.}
\label{tab:PCA-ASUN-Lexi-flip-qmsym}
\resizebox{\columnwidth}{!}{
\begin{tabular}{lcccc}
\toprule
& Flip rate & Median $\alpha$ & p95 $\alpha$ & Not invariant \\
\midrule
PCA & $0.581$ & $57.9$ & $3.1{\times}10^{6}$ & $1000/1000$ \\
ASUN @ $10^{-2}$ & $0.782$ & $54.9$ & $1.5{\times}10^{3}$ & $456/1000$ \\
\our{} @ $10^{-3}$ & $\mathbf{0.208}$ & $\mathbf{1.52}$ & $\mathbf{435}$ & $\mathbf{0/1000}$ \\
\bottomrule
\end{tabular}}
\end{table}

As observed, \textbf{(1)} \our has lower flip rate than baselines and has a lower median amplification rate \textbf{(2)} \our preserves invariance across all the tested QM-sym molecules which the baselines fail to do.

\subsection{End-to-end Learning Experiments}
\label{subsec:qm9_ml}
To test the impact of our method on downstream learning tasks, we evaluate \our against baselines on molecular property prediction and 3D shape classification tasks.

\noindent\textbf{Regression on molecules.} We tested property prediction on QM9 using a standard Transformer architecture (details in Appendix~\ref{app:transformer_architecture}) trained on the regular dataset, on the ASUN-canonicalized dataset, and on the proposed canonicalized representation.

As observed in Table~\ref{tab:qm9_asun_comparison}, \our achieves the lowest error and improves on ASUN on every seed.

\begin{table}[H] 
\centering
\setlength{\tabcolsep}{2pt} 
\renewcommand{\arraystretch}{0.95} 
\caption{QM9 isotropic polarizability prediction (Transformer, Table~\ref{tab:model_hyperparameters}). Test MAE and RMSE over three seeds ($\text{mean} \pm \text{STD}$), on randomly rotated, permuted and translated test molecules (see Appendix~\ref{app:transformer_architecture} for more details).} 
\label{tab:qm9_asun_comparison} 
\begin{tabular}{@{}lcc@{}} 
\toprule 
Representation & MAE & RMSE \\ 
\midrule 
Regular & $0.81{\pm}0.07$ & $1.22{\pm}0.09$ \\ 
ASUN & $0.46{\pm}0.04$ & $0.70{\pm}0.04$ \\ 
\our & $\mathbf{0.40{\pm}0.04}$ & $\mathbf{0.60{\pm}0.06}$ \\ 
\bottomrule 
\end{tabular} 
\end{table}

\noindent\textbf{Classification of 3D point clouds.} Using the ModelNet40, we further demonstrate the superiority of \our against ASUN for point cloud alignment. We sample $512$ points from each cloud and train invariant/equivariant (SchNet \citep{schutt2017schnet} and EGNN \citep{satorras2021en}) and non-invariant (GCN \citep{kipf2017semisupervisedclassificationgraphconvolutional}, GCNII \citep{chen2020simpledeepgraphconvolutional}, SchNet-Pos and EGNN-Pos) models on ASUN versus \our canonicalized data. 

\begin{table}[H]
\centering
\caption{Accuracy of classification (\%, mean $\pm$ STD over three seeds). Best results are in bold.}
\label{tab:results-n512}
\resizebox{\columnwidth}{!}{%
\begin{tabular}{@{}lcccc@{}}
\multicolumn{5}{c}{ModelNet40 ($n=512$)} \\
\cmidrule(lr){1-5}
Model & $z/\SO(3)$ & $\SO(3)/\SO(3)$ & ASUN & \our \\
\midrule
SchNet & $27.8\pm0.5$ & $\bm{28.7\pm0.3}$ & $28.6\pm0.7$ & $28.1\pm0.1$ \\
EGNN & $\bm{34.9\pm0.7}$ & $31.4\pm4.9$ & $32.2\pm0.8$ & $33.4\pm0.1$ \\
\midrule
GCN & $8.2\pm0.6$ & $15.2\pm0.2$ & $17.1\pm2.5$ & $\bm{28.6\pm1.3}$ \\
GCNII & $6.9\pm0.8$ & $19.1\pm1.1$ & $21.9\pm1.3$ & $\bm{29.1\pm2.7}$ \\
SchNet-Pos & $10.5\pm1.1$ & $29.2\pm2.3$ & $30.0\pm0.8$ & $\bm{37.1\pm3.2}$ \\
EGNN-Pos & $10.3\pm1.4$ & $31.5\pm2.3$ & $31.8\pm1.5$ & $\bm{37.0\pm0.5}$ \\
\bottomrule
\end{tabular}%
}
\end{table}
As seen in Table~\ref{tab:results-n512}, \textbf{(1)} performance for invariant models remains largely unchanged, but \our beats ASUN for $4/4$ non-invariant models and produces overall best accuracy of $37.1\%$ across architectures, \textbf{(2)} using \our representation in EGNN-pos/SchNet-pos (which add an embedding of the coordinates to the invariant backbone of EGNN/SchNet) improves upon the respective baseline invariant models.

\section{Conclusion and Future Work}
\label{sec:conclusion}
In this work, we introduced \our, a lexicographic canonicalization method for labeled 3D point clouds that provides a provable canonical representative under permutation and either $\SO(3)$ or $\mathrm{O}(3)$ transformations, while achieving high-precision invariance and a substantially improved runtime over ASUN. Experiments on QM9, QM-sym, ModelNet40 datasets, and controlled perturbation on molecular configurations further show that \our avoids the invariance failures and reduces discontinuity jumps observed in PCA- and ASUN-based approaches, while improving downstream performance for non-invariant architectures. Our canonicalization runs in the worst case $O(n^2 \log(n))$ compared to ASUN's $O(n\log(n))$, yet empirically it is faster on QM9 and ModelNet40, as seen in Table~\ref{tab:runtime-real-data}. Extending our empirical evaluation of combining \our with additional state-of-the-art architectures remains a direction for future work. In particular, we aim to investigate the trade-offs between specialized equivariant architectures and lightweight standard models aided by canonicalization.

\subsection*{AI use statement}
In this work, we used generative AI tools for assistance in implementing the algorithmic methods for the experiments, assist in the writing proofs of theorems and claims by finding errors and refining notation, refining the readability of the manuscript, assistance in generating a synthetic data set of symmetric point clouds in Appendix~\ref{app:symmetric-polyhedra-data}, and provide feedback on the experiment methodology. We have not used generative AI tools to help develop theoretical models or conceptual frameworks, formulate mathematical claims, provide critical ingredients for proving mathematical claims, propose or refine hypotheses,  clean and reformat datasets, support qualitative and thematic data analysis, interpret results, and assisting with translation is not applicable to this work.
Additionally, we used generative AI tools for formatting the tex document, assistance in refining figures and tables, and conducting literature review to survey the state of progress in the field.
We have reviewed all AI-assisted work. We manually inspected LLM-generated code and tested for correctness  by $3$ authors. All writing edits and refinements assisted by LLMs was verified and reviewed by the authors. All proofs refined by LLMs were checked and verified line by line by the authors. We take responsibility for the final content of this work, including text, claims, or artifacts produced with the aid of generative AI.

\subsubsection*{Acknowledgments}
The authors acknowledge London Geometry and Machine Learning (LOGML) Summer School 2026 since this work was initially started as a project during the School. D.M. acknowledges support by the UKRI Centre for Doctoral Training in Algebra, Geometry and Quantum Fields (AGQ), Grant Number EP/Y035232/1. L.B. acknowledges financial support from the European Regional Development Fund for the project ‘Materials for clean energy, advanced sensors and quantum technologies’ (Grant No. PK.1.1.10.0002). S.H. thanks the Gloria and Ken Levy Foundation for a research fellowship.

\bibliographystyle{plainnat} 
\bibliography{references}   

\clearpage
\appendix
\thispagestyle{empty}

\onecolumn

\appendix

\section{Additional Related Work}\label{apdx:add-rlt-wrk}

\textbf{Additional canonicalization methods.} Other approaches learn the canonicalization itself \citep{kaba2023equivariance}. Together with the methods in Section~\ref{sec:related}, these approaches either average over multiple valid frames, resolve only specific ambiguities in the representation, or rely on assumptions that exclude non-trivial stabilizers. In contrast, our method defines a single map that returns a provable canonical representative for every point cloud, including symmetric inputs with non-trivial stabilizers.

\textbf{Theoretical limitations of canonicalization.} \citet{dymlawrence2024impossibility} establish impossibility results for continuous canonicalization under common group actions and illustrate the trade-off using exact but discontinuous constructions, such as lexicographic sorting. Relatedly, \citet{baker2024explicit} show that continuous frames do not exist in general for $\mathrm{O}(n)$ and $\SO(n)$ when $n\geq3$. The difficulty arises particularly at symmetric inputs, where non-trivial stabilizers make the choice of a unique representative inherently ambiguous. Our work, therefore, focuses on constructing an exact representative and characterizing its numerical behavior around these unavoidable discontinuities.

\textbf{Equivariant and invariant representations.} Beyond canonicalization, equivariant models preserve the transformation structure throughout the network, while invariant architectures produce outputs that are unchanged under the group action. Another line of work constructs complete geometric invariants of point clouds using quantities such as point- or simplex-wise distance distributions \citep{p1-widdowson2023recognizing,p1-kurlin2024polynomial,p2-kurlin2023simplexwise,p1-widdowson2022resolving}. Unlike canonicalization, invariant representations generally do not return an oriented representative of the original orbit. Canonicalization therefore provides a complementary strategy, allowing standard non-equivariant architectures to operate directly on a canonical representative.

\section{Algorithms for Canonical Pose Computation}
\label{app:algorithms}
First, we present the two-pointer based algorithm to find the lexicographically minimum orientation of a 2D string. The exact details of the algorithm are given in the main text.For a point cloud $Y=((y_i,\ell_i))_{i=1}^n$, let $L(Y)$
list its original labeled records in the order determined by
$(a_i,y_{i,1},\ldots,y_{i,d})$.
When comparing two candidate sequences for the same input cloud,we compare these rank-coordinate keys lexicographically.

\begin{algorithm}[htb!]
\small
\caption{\textbf{Least cyclic shift of a circular word.}}
\label{alg:linear-cyclic-shift}
\KwIn{Nonempty word $W=(\omega_0,\ldots,\omega_{m-1})$}
\KwOut{Index $s$ such that $\tau_sW=\min_{0\leq r<m}\tau_rW$}
\If{$m=1$}{\KwRet{$0$}\;}
$i\gets0,\; j\gets1,\; t\gets0$\;
\While{$i<m$ \textnormal{and} $j<m$ \textnormal{and} $t<m$}{
    $u\gets\omega_{(i+t)\bmod m},\; v\gets\omega_{(j+t)\bmod m}$\;
    \If{$u=v$}{
        $t\gets t+1$\;
    }
    \ElseIf{$u<v$}{
        $j\gets j+t+1$\;
        \If{$j\leq i$}{$j\gets i+1$\;}
        $t\gets0$\;
    }
    \Else{
        $i\gets i+t+1$\;
        \If{$i\leq j$}{$i\gets j+1$\;}
        $t\gets0$\;
    }
}
\KwRet{$\min\{i,j\}$}\;
\end{algorithm}

Now we give the algorithm to compute the canonicalization of a 3D point cloud that uses Algorithm~\ref{alg:linear-cyclic-shift} after projecting from a chosen axis and chooses the representation with the lexicographically minimum output. The exact details of the algorithm are given in the main text.

\begin{algorithm}[htb!]
\small
\caption{\textbf{Canonicalization of a centered labeled 3D point cloud.}}
\label{alg:three-dimensional-canonicalization}
\KwIn{Centered cloud $P=((\xb_i,\ell_i))_{i=1}^n$, group $G_{3}^{\kappa}$ with $\kappa\in\{\SO,\mathrm O\}$}
\KwOut{Canonical form $\mathcal C_3^\kappa(P)$}
Compute $r_i\gets\|\xb_i\|,\; r_\star\gets\max_i r_i$\;
\If{$r_\star=0$}{\KwRet{zero points ordered by their labels}\;}
Compute the rarest-first ranks $a_i$ of the original labels $\ell_i$\; construct $\Pi(P)$ as in Eq.~\eqref{eq:candidate-poles}\; $D_{\mathrm{best}}\gets\bot$\;
\For{$p\in\Pi(P)$}{
    Choose $A_p\in\SO(3)$ s.t.\ $A_p\xb_p/\|\xb_p\|=\eb_3$\;
    Write $A_p\xb_i=(\ub_i^{(p)},h_i^{(p)})$ and construct
$P_p=((\ub_i^{(p)},\ell_i,h_i^{(p)}))_{i=1}^n$,
retaining the previously computed ranks $a_i$\;
    \eIf{$\kappa=\SO$}{
        $\widehat P_p\gets\mathcal C_2^{\SO}(P_p)$, using Algorithm~\ref{alg:linear-cyclic-shift} to compute $s(P_p)$\;
    }{
        $\widehat P_p\gets\mathcal C_2^{\mathrm O}(P_p)$, using Algorithm~\ref{alg:linear-cyclic-shift} for $P_p$ and $\overline P_p$\;
    }
   Restore each retained height:
$(\widehat u,\ell,h)\mapsto
((\widehat u_1,\widehat u_2,h),\ell)$\;
Order the lifted records by $(a,x,y,z)$, and denote the ordered candidate by $Y_p^\kappa$\;
$D_p^\kappa\gets Y_p^\kappa$\;
\If{$D_{\mathrm{best}}=\bot$ \textnormal{or} $D_p^\kappa<D_{\mathrm{best}}$}{$D_{\mathrm{best}}\gets D_p^\kappa$\;}
}
\KwRet{$D_{\mathrm{best}}$}\;
\end{algorithm}

\section{Proofs of Technical Results}
We use the circular word $W(P)=(\omega_0,\ldots,\omega_{m-1})$
defined in the main text. For $m>0$, let $\tau_sW(P)$ denote the
word read cyclically from position $s$, with indices taken modulo $m$.
Define
\begin{equation}
\label{eq:attributed-least-shift}
s(P)\in
\arg\min_{0\leq s<m}\tau_sW(P).
\end{equation}
The minimum uses lexicographic order. In two dimensions, each letter
is $(a_{i_j},\rho_{i_j},\delta_j)$, where $a_i$ is the rarity rank
of the original label. For a projected three-dimensional cloud,
the height $h_{i_j}$ is appended to each letter.

For $m>0$, let $j_t=i_{(s(P)+t)\bmod m}$.
After selecting the cyclic start, define the orientation key
\begin{equation}
\label{eq:attributed-orientation-key}
B(P)\coloneqq
\left(
(a_{j_t})_{t=0}^{m-1},
(\delta_{(s(P)+t)\bmod m})_{t=0}^{m-1},
(\rho_{j_t})_{t=0}^{m-1}
\right).
\end{equation}
This key compares the complete label-rank sequence first,
then the complete angular-gap sequence, and finally the complete
radial sequence. For a projected three-dimensional cloud,
append the complete height sequence $(h_{j_t})_{t=0}^{m-1}$
as the final component. If $m=0$, no cyclic start is needed
and $B(P)$ is the empty sequence.

We start with formal theorem statements and proofs of the validity of our canonicalization $\mathcal C_d^\kappa(\cdot)$ for $d\in\{2,3\}$ and $\kappa\in\{\SO, \mathrm O\}$
\label{app:proofs}
\begin{lemma}[\textbf{Correctness of 2D canonical maps}]
\label{lem:circular-canonicalization}
\textit{The maps $\mathcal C_2^{\SO}(\cdot)$ and
$\mathcal C_2^{\mathrm O}(\cdot)$ are canonicalizations for
$G_2^{\SO}$ and $G_2^{\mathrm O}$, respectively.}
\end{lemma}
\begin{proof}
We verify that the maps are well defined, invariant under their respective
group actions, and lie in the orbit of the input under the corresponding group actions. Rarity ranks depend only on the multiset of original labels.
They are unchanged by spatial transformations and are permuted
together with their points. Distinct original labels present in
the input have distinct ranks.

\paragraph{Well-definedness:}
Let $I$ be the set of indices of points not at the origin and let $m\coloneqq |I|$. First suppose $m>0$. Let $\phi_i$ be the polar angle of each such point. Order the indices in $I$ lexicographically by
$(\phi_i,a_i,\rho_i)$ and denote them by
$i_0,\ldots,i_{m-1}$. Define the angular gaps by
\begin{align*}
\delta_j=\phi_{i_{j+1}}-\phi_{i_j},\,\forall\,0\leq j<m-1,\quad
\delta_{m-1}=2\pi+\phi_{i_0}-\phi_{i_{m-1}}
\end{align*}
This lets us write $P$ as a circular word:
\begin{equation}
\label{eq:attributed-circular-word}
W(P)=(\omega_0,\ldots,\omega_{m-1}),
\quad
\omega_j=(a_{i_j},\rho_{i_j},\delta_j)
\end{equation}
Such a circular word has an inherent total order, induced by its coordinates. We choose a shift $s(P)$ (see Eq.~\eqref{eq:attributed-least-shift}) which gives the word with the least lexicographic order
Thus, the starting position is selected by repeatedly comparing
$(a,\rho,\delta)$ around the circle and choosing the lexicographic minimum.

Suppose $s$ and $s'$ both achieve the minimum in
\eqref{eq:attributed-least-shift}. Then
$\tau_sW(P)=\tau_{s'}W(P)$.
Their corresponding rank, radius, and angular-gap sequences
therefore agree. Consequently, they also give the same channelwise
orientation key $B(P)$.

Once the first selected point is placed at angle $\pi/2$, the remaining
canonical angles are determined recursively by angular increments. More
precisely, for $j_t=i_{(s+t)\bmod m}$, we have:
\begin{align*}
    \widehat\phi_t=\frac{\pi}{2}+\sum_{r=0}^{t-1}\delta_{(s+r)\bmod m}\pmod{2\pi},
\end{align*}
The empty sum is zero. Thus either minimizing starting point reconstructs
the same non-origin Cartesian records. Their original labels
are determined by their ranks among the labels present in $P$ and points at the origin remain fixed and are included in the final
ordering by $(a_i,\widehat u_{i,1},\widehat u_{i,2})$.
Hence tied starting positions give the same complete output,
and both $\mathcal C_2^{\SO}$ and $B(P)$ are well defined.

\paragraph{Permutation invariance:}
A permutation changes only the indices attached to the points. The subsequent
ordering by $(\phi_i,a_i,\rho_i)$ therefore produces the same
ordered sequence of point records. If two such records are identical, their
relative order is irrelevant because exchanging them does not change the
output. After rotation, all points, including those at the origin, are
ordered by $(a_i,\widehat u_{i,1},\widehat u_{i,2})$.
This ordering also does not depend on the input row indices. Consequently, for every $\pi\in\mathfrak S_n$,
\begin{align*}
W(\pi\cdot P)=W(P),\,
B(\pi\cdot P)=B(P),\,
\mathcal C_2^{\SO}(\pi\cdot P)=\mathcal C_2^{\SO}(P).
\end{align*}

\paragraph{$\SO(2)$ invariance.}

Let $R_\beta\in\SO(2)$ be a rotation by angle $\beta$, and let $P'=R_\beta P$. The radii and attributes are unchanged, while every nonzero point has angle $\phi_i'=\phi_i+\beta\pmod{2\pi}.$ Adding the same angle to every point preserves their circular order. Hence, for some $r$, $W(P')=\tau_rW(P).$ The set of cyclic shifts of a word is unchanged when the word itself is cyclically shifted. Therefore $\min_s\tau_sW(P')=\min_s\tau_sW(P),$ and hence $B(P')=B(P)$.

Choose corresponding minimizing positions for $P$ and $P'$. The selected
points differ in angle by $\beta$. For every index $t$, we have:
\begin{align*}
\widehat\phi_t'&=\phi_{j_t}'-\phi_{j_0}'+\frac{\pi}{2}\pmod{2\pi}\\
&=(\phi_{j_t}+\beta)-(\phi_{j_0}+\beta)+\frac{\pi}{2}\pmod{2\pi}\\
&=\widehat\phi_t.
\end{align*}
The canonical radii, angles, attributes, and row order are therefore
unchanged. Together with permutation invariance, this proves that $\mathcal C_2^{\SO}(g\cdot P)=\mathcal C_2^{\SO}(P),\,\forall g\in G_2^{\SO}.$

The output is obtained by rotating every point through
$\pi/2-\phi_{j_0}$ and then applying a row permutation. Thus $\mathcal C_2^{\SO}(P)\in G_2^{\SO}\cdot P.$ If $m=0$, every point is at the origin, and the result follows directly from
the prescribed ordering of the original labels. Hence $\mathcal C_2^{\SO}(\cdot)$ is a canonicalization for
$G_2^{\SO}$.

\paragraph{$\mathrm O(2)$ invariance.}

Let $F=\operatorname{diag}(1,-1)$ be reflection across the first coordinate
axis, so that $\overline P=FP$. We will use the fact that every $Q\in\mathrm O(2)$ is either
$Q=R_\beta$ or $Q=R_\beta F$ for some $\beta$, where $R_\beta$ denotes rotation in the 2D plane by an angle $\beta$.

First suppose $Q=R_\beta$. By the $\SO(2)$ invariance proved above, we get: $B(QP)=B(P)$ and $\mathcal C_2^{\SO}(QP)=\mathcal C_2^{\SO}(P).$
Since $FR_\beta=R_{-\beta}F$, we also have $\overline{QP}=FR_\beta P=R_{-\beta}\overline P.$
Another application of $\SO(2)$ invariance gives
$B(\overline{QP})=B(\overline P),\quad, \mathcal C_2^{\SO}(\overline{QP})=\mathcal C_2^{\SO}(\overline P).$
Thus a proper rotation leaves both candidates in Eq.~\eqref{eq:circular-canonicalization-o} unchanged.

Now suppose $Q=R_\beta F$. Then we have:
\begin{align*}
    QP=R_\beta\overline P,\quad\overline{QP}=FR_\beta FP=R_{-\beta}P~.
\end{align*}
Consequently, we get
\begin{align*}
B(QP)&=B(\overline P),\quad
\mathcal C_2^{\SO}(QP)=\mathcal C_2^{\SO}(\overline P)~,\\
B(\overline{QP})&=B(P),\quad
\mathcal C_2^{\SO}(\overline{QP})=\mathcal C_2^{\SO}(P)~.
\end{align*}
Therefore, an improper transformation exchanges the two candidates in Eq.~\eqref{eq:circular-canonicalization-o}. Selecting the candidate with the smaller orientation key $B$ is unchanged by this exchange. Hence
$
\mathcal C_2^{\mathrm O}(g\cdot P)=
\mathcal C_2^{\mathrm O}(P)\quad\text{for every }g\in G_2^{\mathrm O}.
$
In any of the branches selected we have, $\mathcal C_2^{\mathrm O}(P)
\in G_2^{\mathrm O}\cdot P.$
Finally, if $B(P)=B(\overline P)$, the reconstruction argument at the
beginning of the proof shows that the two branches give the same canonical
cloud. Therefore $\mathcal C_2^{\mathrm O}(\cdot)$ is well defined and
is a canonicalization for $G_2^{\mathrm O}$.

The same argument applies when heights are retained as additional
attributes unchanged by planar transformations. The angular ordering
then uses $(\phi_i,a_i,\rho_i,h_i)$, the circular letters use
$(a_i,\rho_i,\delta_i,h_i)$, and the complete height sequence is
appended to $B(P)$. The final ordering uses
$(a_i,\widehat u_{i,1},\widehat u_{i,2},h_i)$.
Ranks are still computed from the original labels alone.
If every projected radius vanishes, the planar transformation is
the identity and the retained records are ordered by $(a_i,h_i)$.
\end{proof}
\begin{theorem}[\textnormal{\textbf{Correctness of 3D canonical maps}}]
\label{thm:three-dimensional-canonicalization}
\textit{The maps $\mathcal C_3^{\SO}(\cdot)$ and $\mathcal C_3^{\mathrm O}(\cdot)$ are
canonicalizations for $G_3^{\SO}$ and $G_3^{\mathrm O}$, respectively.}
\end{theorem}

\begin{proof}
We verify that the maps are well defined, invariant under the corresponding
group actions, and return a representative from the appropriate orbit.

Let $r_\star$ and $\Pi(P)$ be as defined in Eq.~\eqref{eq:candidate-poles}. First we suppose that $r_\star=0$. Then every point lies at the origin, and both
maps return the zero points ordered by their labels. This output is unchanged
by rotations, reflections, and row permutations, and differs from $P$ only by
a row permutation. The result therefore holds in this case. Henceforth, assume
$r_\star>0$.

\paragraph{Independence of pole-alignment rotation:}
For a fixed $p\in\Pi(P)$, suppose that $A_p,\widetilde A_p\in\SO(3)$ both align the pole with $\eb_3$. Thus we have $A_p\frac{\xb_p}{\|\xb_p\|}=\widetilde A_p\frac{\xb_p}{\|\xb_p\|}=\eb_3.$ The matrix $T=\widetilde A_pA_p^\top$ belongs to $\SO(3)$ and fixes $\eb_3$. Since $T$ is orthogonal, it preserves the orthogonal complement $\eb_3^\perp$. Hence, with respect to
$\mathbb R^3=\eb_3^\perp\oplus\operatorname{span}\{\eb_3\}$, it has the form
$T=\begin{pmatrix}
H&0\\
0&1
\end{pmatrix}$
for some $H\in\SO(2)$. Therefore, the projected clouds obtained from $A_p$
and $\widetilde A_p$ differ only by the planar rotation $H$, their heights
and labels being identical. The height-augmented result of
Lemma~\ref{lem:circular-canonicalization}, using the unchanged
original-label ranks, therefore gives the same planar output.
Restoring heights and ordering by $(a,x,y,z)$ gives the same
$Y_p^{\SO}$ and $Y_p^{\mathrm O}$ for either choice of $A_p$.

\paragraph{Invariance to $\permn$ and spacial transformations:}
Let $P'=(\pi,Q)\cdot P$, and write
$P'=\{(\xb'_j,\ell'_j)\}_{j=1}^n$. By the group action, $\xb'_{\pi(i)}=Q\xb_i,\,\ell'_{\pi(i)}=\ell_i.$
Since $Q$ is orthogonal, we get $\|\xb'_{\pi(i)}\|=\|Q\xb_i\|=\|\xb_i\|.$
The multiset of original labels is unchanged, so the rarity ranks
satisfy $a'_{\pi(i)}=a_i$ for every $i$.
Thus $P$ and $P'$ have the same maximum radius and the same
smallest rarity rank among the farthest points.
Consequently, $\Pi(P')=\pi(\Pi(P))$, and
$p\mapsto\pi(p)$ is a bijection between the pole sets.
Now, we fix $p\in\Pi(P)$ and let $p'=\pi(p)$. Choose proper rotations
$A_p,A'_{p'}\in\SO(3)$ satisfying $A_p\frac{\xb_p}{\|\xb_p\|}=\eb_3,\,A'_{p'}\frac{\xb'_{p'}}{\|\xb'_{p'}\|}=\eb_3.$ Define $S\coloneqq A'_{p'}QA_p^\top$. Since
$\xb'_{p'}=Q\xb_p$ and orthogonal transformations preserve norms, we get:
\begin{align*}
S\eb_3=A'_{p'}QA_p^\top\eb_3=
A'_{p'}Q\frac{\xb_p}{\|\xb_p\|}=
A'_{p'}\frac{\xb'_{p'}}{\|\xb'_{p'}\|}=
\eb_3~.
\end{align*}
Thus, $S$ is orthogonal and fixes $\eb_3$. Its restriction to
$\eb_3^\perp$ is therefore an element $H\in\mathrm O(2)$, and we can write
$S=
\begin{pmatrix}
H&0\\
0&1
\end{pmatrix}.$
Moreover, because $A_p,A'_{p'}\in\SO(3)$, we have $\det H=\det S=\det(A'_{p'})\det(Q)\det(A_p^\top)=\det Q.$
Consequently, $H\in\SO(2)$ when $Q\in\SO(3)$, and similarly $H\in\mathrm O(2)$ when $Q\in\mathrm O(3)$.

For every $i$, we write $A_p\xb_i=(\ub_i^{(p)},h_i^{(p)})$. Then we have:
\begin{align*}
A'_{p'}\xb'_{\pi(i)}=A'_{p'}Q\xb_i=SA_p\xb_i=
\bigl(H\ub_i^{(p)},h_i^{(p)}\bigr)~.
\end{align*}
Together with $\ell'_{\pi(i)}=\ell_i$ and $a'_{\pi(i)}=a_i$,
this shows that corresponding projected points have identical
heights, original labels, and rarity ranks.
Their planar coordinates differ by $H$, and their rows differ by $\pi$.

If $Q\in\SO(3)$, then $H\in\SO(2)$. Applying Lemma~\ref{lem:circular-canonicalization} therefore gives $D_{\pi(p)}^{\SO}(P')=D_p^{\SO}(P)\,\forall\,p\in\Pi(P).$ Since $p\mapsto\pi(p)$ is a bijection between the pole sets, we get
$\bigl\{D_{p'}^{\SO}(P'):p'\in\Pi(P')\bigr\}=\bigl\{D_p^{\SO}(P):p\in\Pi(P)\bigr\}.$
Taking the lexicographic minimum and using Eq.~\eqref{eq:canonicalization-so3} yields $\mathcal C_3^{\SO}(P')=\mathcal C_3^{\SO}(P).$ If $Q\in\mathrm O(3)$, then $H\in\mathrm O(2)$. The same argument using
$\mathcal C_2^{\mathrm O}(\cdot)$ gives $D_{\pi(p)}^{\mathrm O}(P')=D_p^{\mathrm O}(P)\,\forall\,p\in\Pi(P),$
and hence, by Eq.~\eqref{eq:canonicalization-o3}, we get $\mathcal C_3^{\mathrm O}(P')=\mathcal C_3^{\mathrm O}(P).$
This proves invariance under $G_3^{\SO}$ and $G_3^{\mathrm O}$,
respectively.

\paragraph{Orbit membership.} We fix $p\in\Pi(P)$. Since $\mathcal C_2^{\SO}(\cdot)$ is a canonicalization,
its output is obtained from $P_p$ by a planar rotation $H_p\in\SO(2)$ and a row permutation. Lifting $H_p$ to $\R^3$ gives
$
\widetilde H_p=
\begin{pmatrix}
H_p&0\\
0&1
\end{pmatrix}
\in\SO(3).$
We denote the corresponding 3D transformation as $\widetilde H_pA_p\in\SO(3)$. The operation $L$ introduces at most another row permutation. Therefore, $\exists\,\sigma_p\in\permn$ such that $D_p^{\SO}(P)=(\sigma_p,\widetilde H_pA_p)\cdot P$, implying $D_p^{\SO}(P)\in G_3^{\SO}\cdot P$.

Similarly, the output of $\mathcal C_2^{\mathrm O}(\cdot)$ is obtained using some
$H_p\in\mathrm O(2)$. Its lift to $\R^3$ gven by $\widetilde H_p$ lies in $\mathrm O(3)$,
and therefore $\widetilde H_pA_p\in\mathrm O(3)$. After combining all row
reorderings into a permutation $\sigma_p\in\permn$, we obtain $D_p^{\mathrm O}(P)=(\sigma_p,\widetilde H_pA_p)\cdot P.$ Hence $D_p^{\mathrm O}(P)\in G_3^{\mathrm O}\cdot P$.

The pole set $\Pi(P)$ is nonempty and finite. Consequently, each
lexicographic minimum in Eq.~\eqref{eq:canonicalization-so3} and Eq.~\eqref{eq:canonicalization-o3} is attained by at least one pole. It follows
that $\mathcal C_3^{\SO}(P)\in G_3^{\SO}\cdot P,\, \mathcal C_3^{\mathrm O}(P)\in G_3^{\mathrm O}\cdot P.$
If several poles attain the minimum, their sorted rank-coordinate
keys agree. Since distinct original labels present in $P$ have
distinct rarity ranks, their returned labeled records also agree.
Thus the minimum is well defined.
Points with zero planar projection remain fixed under every
$H\in\mathrm O(2)$, with their heights and labels retained.
If all planar projections vanish, the planar transformation is
the identity and the final ordering is determined by $(a_i,h_i)$.

Thus, each map is invariant under its stated group and returns a point in the
corresponding orbit. By Definition~\ref{def:canonicalization},
$\mathcal C_3^{\SO}(\cdot)$ and $\mathcal C_3^{\mathrm O}(\cdot)$ are canonicalizations
for $G_3^{\SO}$ and $G_3^{\mathrm O}$, respectively.
\end{proof}
Now, we prove correctness of \our algorithm in computing the maps $\mathcal C_3^\kappa(\cdot),\,\kappa\in\{\SO,\mathrm O\}$ in the following theorem.

\begin{theorem}[\textnormal{\textbf{Correctness of \our algorithm}}]
\label{thm:algorithm-correctness}
\textit{Algorithm~\ref{alg:linear-cyclic-shift} returns a least cyclic shift of
$W$. Consequently, when used in the planar constructions Eq.~\eqref{eq:circular-canonicalization-so} and Eq.~\eqref{eq:circular-canonicalization-o}, it computes
$\mathcal C_2^{\SO}(\cdot)$ and $\mathcal C_2^{\mathrm O}(\cdot)$, respectively. Algorithm~\ref{alg:three-dimensional-canonicalization} computes
$\mathcal C_3^{\SO}(\cdot)$ when $\kappa=\SO$ and
$\mathcal C_3^{\mathrm O}(\cdot)$ when $\kappa=\mathrm O$.
Hence the algorithms compute the canonicalizations established in Lemma~\ref{lem:circular-canonicalization} and Theorem~\ref{thm:three-dimensional-canonicalization}.}
\end{theorem}
\begin{proof}
We use exact arithmetic and exact comparisons.
The entries of $W=(\omega_0,\ldots,\omega_{m-1})$ belong to a
totally ordered set. The scan receives a nonempty word.
The case $m=1$ is immediate, so assume $m\geq2$. We extend $W$ periodically by defining, for every $r\in\mathbb Z$, $
\widetilde\omega_r \coloneqq\omega_{r\bmod m}~,~ R_r \coloneqq\bigl(\widetilde\omega_r,\widetilde\omega_{r+1},
\ldots,\widetilde\omega_{r+m-1}\bigr)~.$
Thus $R_r=\tau_{r\bmod m}W$ and $R_{r+m}=R_r$. Let $\mathcal M
\coloneqq
\left\{
r\in\{0,\ldots,m-1\}:
R_r=\min_{0\leq q<m}R_q
\right\}$
be the nonempty set of least starting positions.

Suppose that the words beginning at positions $a$ and $b$ agree for their
first $t$ entries and then satisfy $\widetilde\omega_{a+q}=\widetilde\omega_{b+q}~\text{for}~0\leq q<t~,\,\widetilde\omega_{a+t}<\widetilde\omega_{b+t}$. Then, for every $0\leq h\leq t$, we have $R_{a+h}<R_{b+h}~.$
Indeed, for $0\leq q<t-h$, we get $\widetilde\omega_{a+h+q}=\widetilde\omega_{b+h+q}~,$
whereas at the next position the two words contain
$\widetilde\omega_{a+t}$ and $\widetilde\omega_{b+t}$, respectively.
The first word is therefore lexicographically smaller. Consequently, every
valid starting position among $b,b+1,\ldots,b+t$ is excluded from
$\mathcal M$. The symmetric statement holds when
$\widetilde\omega_{a+t}>\widetilde\omega_{b+t}$.

We now establish the following loop invariants:

\begin{enumerate}
    \item The two candidate indices are distinct.
    \item If $i,j<m$, then $R_i$ and $R_j$ agree in their first $t$ entries.
    \item Every least start below $\max\{i,j\}$ is represented by one of the
    two candidates:
    \begin{align*}
    \mathcal M\cap
    \{0,\ldots,\max\{i,j\}-1\}
    \subseteq
    \{i,j\}~.
    \end{align*}
\end{enumerate}

These properties hold initially because $i=0$, $j=1$, and $t=0$. Suppose
they hold at the beginning of an iteration. If $u=v$, increasing $t$ by one
preserves all three properties.

Suppose that $u<v$. The elimination claim shows that none of
$j,j+1,\ldots,j+t$ belongs to $\mathcal M$. Let
$j'=j+t+1$. If $j'>i$, replacing $j$ by $j'$ skips exactly this excluded
block. If $j'\leq i$, the algorithm instead sets $j=i+1$. Every position
between $j'$ and $i-1$ was already below the previous maximum and, by the
third invariant, was already excluded; the position $i$ itself remains
represented by the other pointer. Thus, the update skips no member of
$\mathcal M$. Resetting $t$ to zero makes the second invariant immediate,
and the updated indices remain distinct. The case $u>v$ is symmetric.
Therefore, all three invariants hold throughout the algorithm.

It remains to examine termination. First we suppose that one pointer reaches
$m$. For example, let $j\geq m$ and $i<m$. Since every valid start lies
below $j$, the third invariant gives $\mathcal M\subseteq\{i\}~.$
Because $\mathcal M$ is non-empty, $i\in\mathcal M$. Hence
$\min\{i,j\}=i$ is a least starting position. The case $i\geq m$ is
symmetric.

Now suppose that the algorithm terminates because $t=m$. Then $i,j<m$ and
$R_i=R_j$. Assume without loss of generality that $i<j$, and set
$d=j-i$. Equality of these two length-$m$ words implies
\begin{align*}
\widetilde\omega_{r+d}
=
\widetilde\omega_r~\forall r\in\mathbb Z~.
\end{align*}
Thus, $d$ is a cyclic period of $W$, and $R_{r+d}=R_r$ for every $r$.

Choose any $r\in\mathcal M$. If $r<j$, the third invariant implies $r=i$.
If $r\geq j$, repeatedly subtracting $d$ gives an index
$r_0\in\{i,\ldots,j-1\}$ such that $R_{r_0}=R_r$. Hence
$r_0\in\mathcal M$, and the third invariant again gives $r_0=i$.
Therefore $R_i$ is a least cyclic shift. Since $R_i=R_j$, both candidates
are valid, and the returned index $\min\{i,j\}=i$ belongs to $\mathcal M$.
The case $j<i$ is symmetric.

We have therefore proved that Algorithm~\ref{alg:linear-cyclic-shift}
returns an index $s$ satisfying
\begin{align*}
\tau_sW=\min_{0\leq r<m}\tau_rW~.
\end{align*}

Next, we connect the procedure to the planar maps. Applied to $W(P)$, it
returns an index satisfying Eq.~\eqref{eq:attributed-least-shift}, and hence
computes the starting position used by $\mathcal C_2^{\SO}(\cdot)$. If two indices
$s$ and $s'$ both attain the minimum, then
$\tau_sW(P)=\tau_{s'}W(P)$. The shifted word records the rarity ranks, radii, angular gaps,
and retained heights when present.
The ranks identify the original labels among those present in $P$. Once the first point
is placed at angle $\pi/2$, the remaining angles are reconstructed by
\begin{align*}
\widehat\phi_t
=
\frac{\pi}{2}
+
\sum_{q=0}^{t-1}\delta_{(s+q)\bmod m}
\pmod{2\pi}~.
\end{align*}
Thus tied least shifts reconstruct the same non-origin records. Including the fixed origin records, converting to Cartesian coordinates, and applying the prescribed final ordering gives the same canonical output.

For $\mathcal C_2^{\mathrm O}$, the scan is applied separately to
$W(P)$ and $W(\overline P)$.
The returned starts determine the channelwise keys $B(P)$ and
$B(\overline P)$ in \eqref{eq:attributed-orientation-key}.
Selecting the candidate with the smaller key gives exactly
\eqref{eq:circular-canonicalization-o}.

If $m=0$, no cyclic start is required.
The standalone zero cloud is ordered by its original labels;
the height-augmented construction retains all heights and orders
the records by $(a_i,h_i)$.
Thus the planar constructions compute
$\mathcal C_2^{\SO}$ and $\mathcal C_2^{\mathrm O}$ exactly.

Finally, consider Algorithm~\ref{alg:three-dimensional-canonicalization}.
If $r_\star=0$, its output agrees with the separately defined all-zero case.
Otherwise, it computes the original-label ranks once and constructs
$\Pi(P)$ from \eqref{eq:candidate-poles}.
For each $p\in\Pi(P)$, it forms
$P_p=((\ub_i^{(p)},\ell_i,h_i^{(p)}))_{i=1}^n$,
retaining the previously computed ranks.
It applies the height-augmented planar construction, restores the
heights, and orders the lifted records by $(a,x,y,z)$ to obtain
$Y_p^\kappa$. It then sets $D_p^\kappa=Y_p^\kappa$.
Candidate comparisons use the associated rank-coordinate keys.

The candidate does not depend on the particular proper rotation chosen to
align the pole: two such rotations differ, after alignment, by an element
of $\SO(2)$ acting on the projection, which is removed by either planar
canonicalization. After processing any collection of poles,
$D_{\mathrm{best}}$ is, by induction, the lexicographic minimum of the
candidates processed so far. Hence, the final output is
\begin{align*}
\min_{p\in\Pi(P)}Y_p^{\SO}
\quad\text{or}\quad
\min_{p\in\Pi(P)}Y_p^{\mathrm O}.
\end{align*}
which is precisely Eq.~\eqref{eq:canonicalization-so3} or Eq.~\eqref{eq:canonicalization-o3}, respectively.

Therefore, the algorithms compute the maps established as canonicalizations
in Lemma~\ref{lem:circular-canonicalization} and Theorem~\ref{thm:three-dimensional-canonicalization}.
\end{proof}
In the following result, we analyze worst case runtime complexity of \our.
\begin{theorem}[\textnormal{\textbf{Runtime analysis of \our algorithm}}]
Assume that the number of labels is fixed and that arithmetic and
label comparisons take constant time. For a word of length $m$, Algorithm~\ref{alg:linear-cyclic-shift} runs in $O(m)$ time and uses
$O(1)$ auxiliary memory after $W$ has been stored. The complete planar
canonicalizations $\mathcal C_2^{\SO}(\cdot)$ and $\mathcal C_2^{\mathrm O}(\cdot)$
require $O(n\log n)$ time and $O(n)$ auxiliary memory for a point cloud $P$ of size $n$. Let $k=|\Pi(P)|$, Algorithm~\ref{alg:three-dimensional-canonicalization}
runs in $O\bigl(k n\log n\bigr)$ time for $n\geq2$. Since $k\leq n$, its worst-case running time is
$O(n^2\log n)$. It uses $O(n)$ auxiliary memory and the bounds hold for both $\SO$ and
$\mathrm O$ variants.
\label{thm:algorithm-runtime}
\end{theorem}

\begin{proof}
We assume that all coordinate dimensions and the number of labels are fixed. Hence each arithmetic operation, label comparison etc. takes
constant time.

We first analyze Algorithm~\ref{alg:linear-cyclic-shift}. At the beginning
of an iteration, define $\Phi\coloneqq i+j+t.$
If the compared entries are equal, the algorithm replaces $t$ by $t+1$,
while $i$ and $j$ remain unchanged. Hence $\Phi$ increases by one.

Suppose instead that $u<v$. Before the optional correction enforcing
$j>i$, the update satisfies $j_{\mathrm{new}}\geq j+t+1,\,t_{\mathrm{new}}=0.$
Consequently, we get:
\begin{align*}
\Phi_{\mathrm{new}}-\Phi
&=
i+j_{\mathrm{new}}+t_{\mathrm{new}}-(i+j+t)\\
&\geq
(j+t+1)-j-t
=
1~.
\end{align*}
The optional assignment $j\gets i+1$ can only increase $j$ further. The
case $u>v$ is symmetric. Therefore, $\Phi$ increases by at least one in
every iteration.

Whenever the loop condition is satisfied, $i<m$, $j<m$, and $t<m$.
Therefore, $\Phi=i+j+t<3m.$ Since $\Phi$ is integer-valued and increases by at least one per iteration, the loop performs at most $O(m)$ iterations. Each iteration uses a constant number of arithmetic operations, comparisons, and assignments. Thus, the running time is $O(m)$.

After $W$ has been stored, the procedure maintains only the indices
$i,j,t$ and the two currently compared entries $u,v$. Because each word
entry has fixed size, this requires $O(1)$ auxiliary memory. Storing the word
itself requires $O(m)$ memory, but this storage is excluded from the stated
auxiliary bound for the scan.

We next consider a planar cloud with $n$ points.
Computing the rarest-first ranks of the original labels takes
$O(n\log n)$ time. Let $m=|\{i:\rho_i>0\}|$.
Computing the radii and the angles of non-origin points takes
$O(n)$ time. Ordering those points by $(\phi_i,a_i,\rho_i)$
costs $O(n\log n)$ time; appending a retained height does not
change this bound. The cases $n\leq1$ take constant time.

Once this ordering has been computed, constructing the angular increments
and the circular word costs $O(m)$. By the first part of the proof, selecting
its least cyclic shift also costs $O(m)$. Constructing the canonical angles, transforming the points, and
forming the orientation key each require $O(n)$ time.
Sorting the final Cartesian records by $(a_i,x_i,y_i)$,
with height appended when retained, costs $O(n\log n)$.
The standalone all-zero case only sorts the original labels,
with the same bound. Consequently, $T_2^{\SO}(n)=O(n\log n)$.
The $\mathrm O(2)$ construction repeats the circular construction for the
reflected cloud and compares the two resulting keys. This performs at most
a constant multiple of the work in the $\SO(2)$ case, so $T_2^{\mathrm O}(n)=O(n\log n).$
The angular ordering, circular word, reflected candidate, and output can all
be stored using arrays of length at most $n$. Thus either planar construction
uses $O(n)$ auxiliary memory.

Finally, consider Algorithm~\ref{alg:three-dimensional-canonicalization}
for $r_\star>0$.
The original-label ranks are computed once in $O(n\log n)$ time.
Given these ranks, computing the radii, maximum radius, minimum
admissible rank, and pole set $\Pi(P)$ takes $O(n)$ time.
Let $k=|\Pi(P)|$. The planar constructions reuse these ranks.

Fix a pole $p\in\Pi(P)$. Since the ambient dimension is fixed, constructing
$A_p$ takes constant time. Applying $A_p$ to all points and constructing
$P_p$ takes $O(n)$ time. By the planar bound, canonicalizing $P_p$ takes
$O(n\log n)$ time. Restoring the heights takes $O(n)$ time.
Ordering the lifted records by $(a,x,y,z)$ to obtain
$Y_p^\kappa=D_p^\kappa$ costs $O(n\log n)$.
Comparing its rank-coordinate key with that of the current best
candidate takes $O(n)$ time in the worst case. The total work for one pole is
therefore $O(n)+O(n\log n)+O(n)+O(n\log n)+O(n)=O(n\log n).$

Summing over all $k$ poles and including the initial rank computation
gives
\begin{align*}
T_3(n,k)
=
O(n\log n)+k\,O(n\log n)
=
O(kn\log n).
\end{align*}
 Since $k\leq n$, we get $T_3(n,k)=O(n^2\log n)$
in the worst case. If the admissible pole is unique, the sharper bound is
$O(n\log n)$.

The all-zero case requires only ordering the labels and therefore costs
$O(n\log n)$ time, which is covered by the stated worst-case bound.

For memory, the algorithm only needs to retain the input, the current projected
cloud, its circular representation, the current lifted candidate, and the
best candidate found so far. Each contains $O(n)$ records. Standard sorting
workspaces also use at most $O(n)$ memory. Because candidates are processed
one pole at a time rather than stored simultaneously, the auxiliary memory
is $O(n)$ rather than $O(kn)$.

The $\mathrm O(2)$ and $\mathrm O(3)$ variants evaluate one additional reflected planar
orientation. This multiplies the planar work by at most a constant and does
not change any asymptotic time or memory bound.
\end{proof}

Now we proceed to show that \our is Lipschitz almost everywhere. We first define some notations. We fix the label sequence $\ab=(\ab_1,\ldots,\ab_n)$ and let 
$$\mathcal X_{\ab}=\left\{((\xb_i,\ab_i))_{i=1}^n:\sum_{i=1}^n\xb_i=0\right\}$$  be the corresponding space of centered point clouds. For $P=((\xb_i,\ab_i))_{i=1}^n$ and
$P'=((\xb_i',\ab_i))_{i=1}^n$ in $\mathcal X_{\ab}$, we define 
\begin{align*}
d_\infty(P,P')\coloneqq
\max_{1\leq i\leq n}\|\xb_i-\xb_i'\|_2,\quad\mathcal B_\infty(P_0,\delta)\coloneqq
\{P\in\mathcal X_{\ab}:d_\infty(P,P_0)<\delta\}.
\end{align*}
Thus $d_\infty$ measures the largest displacement of corresponding
points, without optimizing over rotations or permutations while $\mathcal B_\infty(P_0,\delta)$ contains the centered clouds in which no point of $P_0$ is displaced by $\delta$ or more.

For this theorem, $\widehat{\mathcal C}_{3,\eta}^{\kappa}$
denotes the construction evaluated in real arithmetic, retaining
its grid comparisons and deterministic selection rules, with
unbounded integer comparison keys.
For $\eta=0$, it denotes the exact construction without quantization.
The theorem does not account for floating-point roundoff or
machine-range restrictions.

\thmrobustness*

\begin{proof}
Fix $\eta\geq0$. We equip $\mathcal X_{\ab}$ with its intrinsic
Lebesgue measure as a finite-dimensional Euclidean space. We first prove that there is a countable collection of open sets
$\{U_j\}_{j\geq1}$ such that their complement has a measure of zero and, on each
$U_j$, the algorithm makes the same choices. 

For a given $n\in\N$, there is a finite upper bound $N(n)$ on the number of comparisons made in an execution of \our. We prove inductively that, after each decision, $\mathcal X_{\ab}$ is covered, up to a set of measure zero, by countably many connected open sets on which all decisions made so far are fixed and all continuously varying quantities computed so far are real-analytic. 

If $n=1$, the centered output is constant.
Assume $n\geq2$ and exclude
\begin{align*}
E_0
\coloneqq
\bigcup_{i=1}^n
\{P\in\mathcal X_{\ab}:\xb_i=0\}.
\end{align*}
Each set in this union is a proper linear subspace, so
$\mu(E_0)=0$. Every possible pole is nonzero outside $E_0$.

We use the following general finite-decision argument. Consider a deterministic procedure with finite decision depth satisfying two properties. First, between consecutive decisions, its continuously varying quantities are updated by real-analytic maps on open domains. Second, each decision is determined by the position of a real-analytic scalar quantity $h(P)$ relative to a locally finite set of thresholds $T\subset\R$.

Let $U$ be a connected open region on which all preceding decisions are
fixed and all quantities computed so far are real-analytic. If the next
decision compares an analytic scalar $h(P)$ with thresholds $T$, define
\begin{align*}
S_U\coloneqq
\!\!\!\!\!\!\!\bigcup_{\substack{c\in T\\h-c\not\equiv0\text{ on }U}}
\!\!\!\!\!\!\!\{P\in U:h(P)=c\}.
\end{align*}
For each $c$ in this union, $h-c$ is a nonzero real-analytic function, so
$\{h=c\}$ has Lebesgue measure zero \citep{mityagin2020}. Since $T$ is locally finite,
$S_U$ is also null. If $h-c\equiv0$ on $U$, the deterministic tie rule
makes the same choice throughout $U$. Hence, the next decision is fixed on
each connected component of $U\setminus S_U$. These components are open
and countable.

The induction gives the claimed partition provided the intervening
operations are defined and analytic on regions covering the input
space up to a null set. We verify this condition for \our.

Polynomial operations, sine and cosine are real analytic.
Square roots and reciprocals are analytic on their positive and
nonzero domains, respectively.
For an analytic vector $v(P)$ on a connected region $U$, either
$v\equiv0$, in which case its norm is constant, or
$\|v(P)\|_2^2$ is a nonzero analytic function.
In the latter case, its zero set is null; away from this set,
$\|v(P)\|_2=\sqrt{\|v(P)\|_2^2}$ is analytic.

Pole normalization is defined outside $E_0$.
For a unit pole direction $u$, the coordinate helper $e$ is chosen
using the smallest absolute coordinate of $u$, giving
$\|e\times u\|_2^2\geq2/3$.
A selected planar reference has nonzero projection; if no reference
is selected, the planar transformation is the identity.
Angular coordinates are analytic on local charts away from the
origin. Refining these charts by coordinate signs uses the same
scalar-comparison argument above.

These refinements discard only null sets and leave countably many
open regions on which the geometric updates are analytic.
Identically zero quantities are retained as constant functions.

The scale $R(P)=\max_i\|\xb_i\|_2$ is computed by comparing
the polynomial quantities $\|\xb_i\|_2^2$.
Refine the preceding regions using these comparisons.
On each resulting region, a maximizing index $i_\star$ is fixed.
Outside $E_0$, the function
$R(P)=\sqrt{\|\xb_{i_\star}\|_2^2}$ is positive and real analytic.
Consequently, $t(P)/R(P)$ is real analytic whenever $t(P)$ is
real analytic on that region. Identically tied maxima are handled
by a fixed deterministic choice of index.

For $\eta>0$, define
$Q_\eta(s)=\lfloor s/\eta+\frac12+\beta\rfloor$,
where $\beta=(\sqrt5-2)/2$.
The implementation bins each spatial scalar $t(P)$, including
radii, heights, and Cartesian coordinates, as
\begin{align*}
Q_\eta\!\left(\frac{t(P)}{R(P)}\right)
=
\left\lfloor
\frac{t(P)}{\eta R(P)}+\frac12+\beta
\right\rfloor,
\end{align*}
while each angular quantity $\theta(P)$ is binned as
$Q_\eta(\theta(P))$.
The same full-cloud scale $R(P)$ is retained throughout
the three-dimensional construction.

Thus each binning decision applies $Q_\eta$ to an analytic
scalar $h(P)$, either $t(P)/R(P)$ or $\theta(P)$.
Its bin can change only when
\begin{align*}
h(P)=\left(k+\frac12-\beta\right)\eta,
\qquad k\in\mathbb Z.
\end{align*}
These thresholds are fixed and locally finite.
Once the bins are fixed, comparisons of their integer keys
are fixed as well. The remaining exact comparisons use
analytic differences with threshold zero.
For $\eta=0$, quantization is omitted.

Thus, the preceding argument
applies to \our and yields connected open sets $\{U_j\}_{j\geq1}$ and
real-analytic maps $F_j(\cdot)$ such that
\begin{align*}
\mu\left(
\mathcal X_{\ab}\setminus\bigcup_{j\geq1}U_j
\right)=0,\,
\left.\widehat{\mathcal C}_{3,\eta}^\kappa(\cdot)\right|_{U_j}=F_j(\cdot).
\end{align*}

Let $E_\eta\coloneqq
\mathcal X_{\ab}\setminus\bigcup_{j\geq1}U_j$. Fix
$P_0\notin E_\eta$ and choose $U_j$ containing $P_0$. Since $U_j$ is open,
there exists $\delta_\eta(P_0)>0$ such that $K\coloneqq\overline{\mathcal B_\infty(P_0,\delta_\eta(P_0))}\subset U_j.$ For $V=(v_1,\ldots,v_n)$, define
$\|V\|_{\infty}\coloneqq\max_i\|v_i\|_2$.
Using its induced operator norm, since $F_j(\cdot)$ is real-analytic and $K$ is compact, we define
\begin{align*}
L_\eta(P_0)
\coloneqq
\sup_{Q\in K}
\|DF_j(Q)\|_{\infty}
<\infty.
\end{align*} 
The mean-value inequality gives, for every $P,P'\in K$,
\begin{align*}
\max_{1\leq i\leq n}
\|F_j(P)_i-F_j(P')_i\|
\leq
L_\eta(P_0)d_\infty(P,P').
\end{align*}
The row order and labels are fixed on $U_j$, so pairing the corresponding rows
is an admissible EMD matching. Therefore,
\begin{align*}
\mathrm{EMD}\left(
\widehat{\mathcal C}_{3,\eta}^\kappa(P),
\widehat{\mathcal C}_{3,\eta}^\kappa(P')
\right) \leq \frac1n\sum_{i=1}^n
\|F_j(P)_i-F_j(P')_i\|\leq L_\eta(P_0)d_\infty(P,P').
\end{align*}
Thus the local Lipschitz bound holds outside the null set $E_\eta$.
Setting $P'=P_0$ proves continuity almost everywhere.
\end{proof}

\section{Additional Experimental Results and Details}

\paragraph{Licenses of datasets and baseline code.}
QM9 and QM-sym are released under the CC0 1.0 Universal
Public Domain Dedication, as specified in their original
dataset records.\footnote{\url{https://doi.org/10.6084/m9.figshare.978904_D12}}
\footnote{\url{https://doi.org/10.6084/m9.figshare.9638093}}
ModelNet40 is provided for academic research only, with
copyright over the CAD models retained by their original
creators.\footnote{\url{https://modelnet.cs.princeton.edu/download.html}}
The ASUN implementation is publicly available in the authors'
repository accompanying their ICML 2024 paper, yet we did not
identify an explicit software license in that
repository.\footnote{\url{https://github.com/Utah-Math-Data-Science/alignment}}
\label{app:expt}
\subsection{QM9 Property Prediction Experimental Details}
\label{app:transformer_architecture}
In the transformer model, each molecule is represented as a sequence of atom tokens. For atom $i$, its atomic-number embedding is added to an embedding of its Cartesian coordinates:
\begin{equation*}
    \mathbf{h}_i=\operatorname{Emb}(z_i)+ \operatorname{MLP}_{\mathrm{coord}}(\mathbf{x}_i)~,\qquad \mathbf{x}_i \in \mathbb{R}^{3}~.
\end{equation*}
A learned \texttt{[CLS]} token is prepended to the sequence, and padding atoms
are excluded using an attention mask. The tokens are processed by a standard transformer encoder without positional encodings, making the architecture invariant to permutations of the atom order. The encoded \texttt{[CLS]} token is passed to an MLP that predicts the molecular property. The model is not intrinsically rotation invariant or equivariant; instead, the canonicalized coordinates provide a consistent reference frame.

Table~\ref{tab:model_hyperparameters} shows the final selected hyperparameters, all architectural and optimization settings are held fixed when comparing regular, ASUN-canonicalized, and our canonicalized coordinates. After selecting the final configurations, models are evaluated using three random seeds.

Each molecule is given one random rotation (or reflection), atom permutation and translation, fixed across epochs. ASUN and \our{} canonicalize this randomly posed copy in training, validation and test, so they cannot benefit from the stored QM9 coordinates, which share a common, non-random orientation. The regular representation is trained and validated on the stored coordinates and tested on the randomly posed molecules, as in the $z/\SO(3)$ setting of ModelNet40. This is because QM9 molecules share a common orientation \citep{not-aligned}, so random rotations are applied at test time to measure the impact of canonicalization versus non-invariance on this dataset. Testing the canonicalized models on a second, independent random pose of every test molecule leaves their error unchanged up to $10^{-3}$ (seed 0: ASUN $0.428$ and $0.429$, \our{} $0.406$ on both poses).

\begin{table}[htb!]
\centering
\caption{Final architecture and training hyperparameters selected for the Transformer model.}
\label{tab:model_hyperparameters}
\begin{tabular}{ll}
\toprule
Hyperparameter & Transformer \\
\midrule

Model dimension
    & $128$ \\
Transformer layers
    & $6$ \\
Attention heads
    & $8$ \\
Feed-forward dimension
    & $512$ \\
Encoder activation
    & GELU \\
Coordinate-MLP activation
    & SiLU\\
Trainable parameters
    & $1,224,961$ \\
\midrule
Train/val/test split
    & $100,000$/$10,000$/$20,831$\\
Batch size
    & $256$\\
Maximum epochs
    & $200$\\
Optimizer
    & AdamW\\
Learning rate
    & $3\times10^{-4}$ \\
Weight decay
    & $10^{-5}$\\
Early-stopping patience
    & $40$ epochs \\
Learning-rate scheduler
    & ReduceLROnPlateau \\
Scheduler factor/patience
    & $0.5$/$5$ epochs \\
\bottomrule
\end{tabular}
\end{table}

\subsection{ModelNet40 Classification Experimental Details}
\label{app:modelnet_architecture}
Each ModelNet40 shape is represented by $n\in\{200,512,1024\}$ points sampled uniformly (area-weighted) from its mesh surface, centered and scaled so that $\max_i\|\mathbf{x}_i\|_\infty=1$. Every model operates on the $k$-nearest-neighbor graph ($k=20$, no self-loops) of the posed points, recomputed for each batch, so that edges always refer to the posed coordinates. We use two families of models, following \citet{baker2024explicit}. The non-invariant models take the coordinates as node features, so their output depends on the pose: GCN \citep{kipf2017semisupervisedclassificationgraphconvolutional} and GCNII \citep{chen2020simpledeepgraphconvolutional} ($\alpha=0.1$, $\theta=0.5$) with coordinate inputs and batch normalization, and SchNet-Pos and EGNN-Pos, which add a linear embedding of the coordinates to the node features of SchNet and EGNN. The invariant controls see only relative geometry: SchNet \citep{schutt2017schnet} (interaction blocks on edge lengths expanded in $50$ Gaussians with cutoff $2.0$, from a shared learned node embedding) and EGNN \citep{satorras2021en} (message passing on squared edge lengths without coordinate updates, nodes initialized from their distance to the centroid). All models end with a readout that concatenates mean and max pooling over the nodes, followed by a two-layer MLP classifier over the $40$ classes.

All models and representations share one configuration (Table~\ref{tab:modelnet_hyperparameters}); no hyperparameter is tuned per representation. Training uses a fixed number of epochs, and test accuracy is reported at the epoch with the best validation accuracy. We report the mean and standard deviation over three seeds at every $n$. In the $z/\SO(3)$ setting, models are trained on the upright shapes and tested on randomly rotated ones; in all other settings, each shape is given one fixed random rotation (and point permutation), and canonical representations canonicalize that rotated copy.

\begin{table}[htb!]
\centering
\caption{Architecture and training hyperparameters for the ModelNet40 models, shared by every model and representation.}
\label{tab:modelnet_hyperparameters}
\begin{tabular}{ll}
\toprule
Hyperparameter & Value \\
\midrule
Hidden dimension & $16$  \\
Layers & $4$ \\
Dropout & $0.5$  \\
$k$-NN graph & $k=20$ \\
Readout & concatenated mean and max pooling, then a two-layer MLP \\
Trainable parameters & GCN $2{,}216$; GCNII $2{,}424$; \\
 & SchNet $8{,}776$; SchNet-Pos $8{,}840$; \\
 & EGNN $7{,}832$; EGNN-Pos $7{,}864$ \\
\midrule
Train/val/test split & $8{,}859$/$984$/$2{,}468$ \\
Batch size & $128$ \\
Epochs & $60$ (no early stopping) \\
Optimizer & Adam \\
Learning rate & $10^{-2}$  \\
Weight decay & $5\times10^{-4}$  \\
Learning-rate scheduler & none \\
Loss & cross-entropy \\
Model selection & best validation accuracy \\
\bottomrule
\end{tabular}
\end{table}

\newpage
\subsection{Additional Results on ModelNet40 Shape Prediction}
The ModelNet40 dataset consists of $40$ categorized CAD meshes, with 9,843 training and 2,468 testing samples. We sample $n\in\{200,512,1024\}$ points from the mesh and use the sampled points to predict properties of the point cloud. We train state-of-the-art invariant/equivariant SchNet \citep{schutt2017schnet} and EGNN \citep{satorras2021en} architectures. Upon adding positional encodings, the invariance breaks, which we denote as SchNet-Pos and EGNN-Pos. We also train non-invariant GCN and GCNII. The main paper reports the accuracy for $n=512$ sampled points, and here we do the same for $n=200, 1024$ in Table~\ref{tab:results-n200-n1024}.

\begin{table}[H]
\centering
\resizebox{\textwidth}{!}{%
\begin{tabular}{@{}l cccc cccc@{}}
\toprule
& \multicolumn{4}{c}{$n=200$} & \multicolumn{4}{c}{$n=1{,}024$} \\
\cmidrule(lr){2-5} \cmidrule(lr){6-9}
Model & $z/\SO(3)$ & $\SO(3)/\SO(3)$ & ASUN & \our & $z/\SO(3)$ & $\SO(3)/\SO(3)$ & ASUN & \our \\
\midrule
SchNet & $27.6\pm0.8$ & $\bm{28.2\pm0.8}$ & $28.1\pm0.4$ & $27.8\pm1.2$ & $27.6\pm0.2$ & $\bm{28.2\pm1.1}$ & $27.9\pm0.1$ & $27.5\pm1.6$ \\
EGNN & $33.9\pm1.7$ & $33.3\pm0.8$ & $\bm{36.3\pm0.1}$ & $32.1\pm1.1$ & $25.7\pm5.3$ & $\bm{30.8\pm2.9}$ & $27.0\pm2.5$ & $28.9\pm2.1$ \\
\midrule
GCN & $7.4\pm0.3$ & $14.7\pm0.5$ & $16.6\pm1.5$ & $\bm{25.3\pm1.3}$ & $8.9\pm0.6$ & $18.9\pm0.6$ & $20.7\pm1.4$ & $\bm{27.1\pm1.4}$ \\
GCNII & $7.7\pm0.9$ & $21.1\pm1.1$ & $24.4\pm0.4$ & $\bm{29.5\pm2.1}$ & $7.2\pm1.8$ & $20.8\pm1.1$ & $25.2\pm0.3$ & $\bm{27.3\pm2.8}$ \\
SchNet-Pos & $10.7\pm1.0$ & $28.6\pm0.4$ & $29.9\pm0.6$ & $\bm{33.1\pm1.8}$ & $10.2\pm0.7$ & $27.7\pm1.6$ & $31.3\pm1.7$ & $\bm{33.5\pm3.1}$ \\
EGNN-Pos & $9.8\pm0.6$ & $28.9\pm3.1$ & $30.3\pm1.0$ & $\bm{35.3\pm2.9}$ & $10.0\pm1.0$ & $26.8\pm0.6$ & $30.1\pm2.6$ & $\bm{34.6\pm1.9}$ \\
\bottomrule
\end{tabular}%
}
\caption{Accuracy of classification on ModelNet40 for $n=200$ and $n=1{,}024$ (\%, mean $\pm$ STD over three seeds). Best results are in bold.}
\label{tab:results-n200-n1024}
\end{table}

As observed from the table, \textbf{(1)} in the shape classification task, \our gives better performance than ASUN for $4/4$ non-invariant architectures for both the mesh of 200 and 1024 points, \textbf{(2)} the invariant SchNet and EGNN do not benefit from canonicalization: SchNet stays within one point across representations, and EGNN is best with ASUN at $n=200$ and without canonicalization at $n=1{,}024$. Intuitively, \our computes a \textit{smoother} canonicalization map than ASUN as observed from the other experiments, which can be used to explain the superior empirical performance.

\subsection{Optimal Numerical Tolerance}
\label{app:tolerance-ASUN-PCA}
All the canonicalization algorithms deal with $\mathrm{O}(3)$ invariance here. Each one of the baseline algorithms (PCA and ASUN), as well as \our, implement a tolerance hyperparameter in a different manner. For PCA, the tolerance does not have any effect against its eigenvalue degeneracy, which is its main failure point. For ASUN, it has two main different effects: it sets the minimum sensitivity for interatomic distances (which is the dominant effect and the one that governs the varying behavior of the algorithm at different tolerances) and it sets a threshold for angular degeneracy (which is the reason why a higher tolerance fails at more configurations). \our, consistently and at every step of the algorithm, treats as equal two lengths within $\eta R$ of each other ($R$ the largest distance from the centroid) and two angles within $\eta$ radians, so that one tolerance means the same on clouds of any size.

To run a fair set of experiments, we compare every algorithm against the 18 adversarial point clouds under a sweep of tolerances, recording the flip rate (amplification $>10$) at perturbations from $10^{-7}$ to $10^{-2}$ and 12 random rotations for each. The results can be seen in Table~\ref{tab:tol-sweep-zoo}. PCA has no tolerance, so its entries are the same in every row. ASUN raises errors on point clouds for tolerances $10^{-1}$ and $10^{-2}$; we choose a tolerance of $10^{-2}$ as it has a good trade-off between flip rate and an acceptable number of failures. \our has a decreasing flip rate as \texttt{tol}$\uparrow1$, however that hurts runtime; it stays invariant on every QM-sym molecule at the $0.1$~\AA{} threshold at every \texttt{tol}. Thus, to achieve a trade-off, we choose a tolerance of $10^{-3}$ for most of the experiments. \label{app:just-tolerance}

\begin{table}[H]
\centering
\caption{Tolerance sweep across the $18$ labeled point clouds. Time is the mean per-cloud canonicalization time in ms (median over 10 calls per cloud). Flip is the fraction of perturbations whose canonical form moved more than $10\times$ further than the input did, pooled over six noise magnitudes ($\sigma = 10^{-7},\ldots,10^{-2}$ of the cloud's RMS radius) and twelve draws per cloud. Non-inv.\ is the fraction of clouds whose random rotations give different canonical forms. Per method and column, bold marks the best \texttt{tol} and underline the second best (times compared to two significant figures); $*$ marks a \texttt{tol} at which the method fails to return a canonicalization for some inputs. \our{} ranks labels rarest first. PCA has no tolerance (one setting, repeated in every row). Best results are in bold.}
\label{tab:tol-sweep-zoo}
\begin{tabular}{@{}l rrr rrr rrr@{}}
\toprule
& \multicolumn{3}{c}{PCA} & \multicolumn{3}{c}{ASUN} & \multicolumn{3}{c}{\our} \\
\cmidrule(lr){2-4} \cmidrule(lr){5-7} \cmidrule(lr){8-10}
\texttt{tol} & Time & Flip & Non-inv. & Time & Flip & Non-inv. & Time & Flip & Non-inv. \\
\midrule
$10^{-9}$ & $\mathbf{0.184}$ & $\mathbf{0.762}$ & $\mathbf{1.000}$ & $\mathbf{5.838}$ & $0.418^{*}$ & $\mathbf{0.353}$ & $\mathbf{2.051}$ & $0.623$ & $\mathbf{0.000}$ \\
$10^{-8}$ & $\mathbf{0.184}$ & $\mathbf{0.762}$ & $\mathbf{1.000}$ & $\underline{5.929}$ & $0.418^{*}$ & $\mathbf{0.353}$ & $\mathbf{2.096}$ & $0.623$ & $\mathbf{0.000}$ \\
$10^{-7}$ & $\mathbf{0.184}$ & $\mathbf{0.762}$ & $\mathbf{1.000}$ & $5.982$ & $0.418^{*}$ & $\mathbf{0.353}$ & $\mathbf{2.082}$ & $0.600$ & $\mathbf{0.000}$ \\
$10^{-6}$ & $\mathbf{0.184}$ & $\mathbf{0.762}$ & $\mathbf{1.000}$ & $6.001$ & $0.408^{*}$ & $\mathbf{0.353}$ & $\mathbf{2.084}$ & $0.503$ & $\mathbf{0.000}$ \\
$10^{-5}$ & $\mathbf{0.184}$ & $\mathbf{0.762}$ & $\mathbf{1.000}$ & $5.987$ & $0.393^{*}$ & $\mathbf{0.353}$ & $\mathbf{2.094}$ & $0.404$ & $\mathbf{0.000}$ \\
$10^{-4}$ & $\mathbf{0.184}$ & $\mathbf{0.762}$ & $\mathbf{1.000}$ & $\underline{5.886}$ & $0.375^{*}$ & $\mathbf{0.353}$ & $\mathbf{2.088}$ & $0.297$ & $\mathbf{0.000}$ \\
$10^{-3}$ & $\mathbf{0.184}$ & $\mathbf{0.762}$ & $\mathbf{1.000}$ & $\underline{5.890}$ & $0.358^{*}$ & $\mathbf{0.353}$ & $\mathbf{2.058}$ & $0.196$ & $\mathbf{0.000}$ \\
$10^{-2}$ & $\mathbf{0.184}$ & $\mathbf{0.762}$ & $\mathbf{1.000}$ & $\underline{5.930}$ & $\underline{0.356}^{*}$ & $\mathbf{0.353}$ & $\mathbf{2.066}$ & $\underline{0.103}$ & $\mathbf{0.000}$ \\
$10^{-1}$ & $\mathbf{0.184}$ & $\mathbf{0.762}$ & $\mathbf{1.000}$ & $\underline{5.916}$ & $\mathbf{0.345}^{*}$ & $\mathbf{0.353}$ & $\underline{2.246}$ & $\mathbf{0.001}$ & $\mathbf{0.000}$ \\
\bottomrule
\end{tabular}%

\end{table}

We also do similar tolerance sweeps for the symmetric molecule dataset QM-sym, the results of which we show below in Table~\ref{tab:tol-sweep-qmsym}.

\begin{table}[H]
\centering
\caption{Tolerance sweep across $1{,}000$ QM-sym molecules Time is the mean per-cloud canonicalization time in ms (median over 10 calls per cloud). Flip is the fraction of perturbations whose canonical form moved more than $10\times$ further than the input did, pooled over six noise magnitudes ($\sigma = 10^{-7},\ldots,10^{-2}$ of the cloud's RMS radius) and twelve draws per cloud. Non-inv.\ is the fraction of molecules with one of $12$ random views more than $0.1$~\AA{} (EMD) from the canonical form. Per method and column, bold marks the best \texttt{tol} and underline the second best (times compared to two significant figures); $*$ marks a \texttt{tol} at which the method fails to return a canonicalization for some inputs. \our{} ranks labels rarest first. PCA has no tolerance (one setting, repeated in every row). Best results are in bold.}
\label{tab:tol-sweep-qmsym}
\begin{tabular}{@{}l rrr rrr rrr@{}}
\toprule
& \multicolumn{3}{c}{PCA} & \multicolumn{3}{c}{ASUN} & \multicolumn{3}{c}{\our} \\
\cmidrule(lr){2-4} \cmidrule(lr){5-7} \cmidrule(lr){8-10}
\texttt{tol} & Time & Flip & Non-inv. & Time & Flip & Non-inv. & Time & Flip & Non-inv. \\
\midrule
$10^{-9}$ & $\mathbf{0.206}$ & $\mathbf{0.581}$ & $\mathbf{1.000}$ & $38.6$ & $0.897$ & $\mathbf{0.276}$ & $\mathbf{1.349}$ & $0.534$ & $\mathbf{0.000}$ \\
$10^{-8}$ & $\mathbf{0.206}$ & $\mathbf{0.581}$ & $\mathbf{1.000}$ & $\underline{38.1}$ & $0.897$ & $\mathbf{0.276}$ & $\underline{1.388}$ & $0.534$ & $\mathbf{0.000}$ \\
$10^{-7}$ & $\mathbf{0.206}$ & $\mathbf{0.581}$ & $\mathbf{1.000}$ & $\underline{37.9}$ & $0.896$ & $\underline{0.297}$ & $1.690$ & $0.519$ & $\mathbf{0.000}$ \\
$10^{-6}$ & $\mathbf{0.206}$ & $\mathbf{0.581}$ & $\mathbf{1.000}$ & $\underline{37.8}$ & $0.875$ & $0.468$ & $1.795$ & $0.457$ & $\mathbf{0.000}$ \\
$10^{-5}$ & $\mathbf{0.206}$ & $\mathbf{0.581}$ & $\mathbf{1.000}$ & $\underline{37.6}$ & $0.861$ & $0.413$ & $1.802$ & $0.373$ & $\mathbf{0.000}$ \\
$10^{-4}$ & $\mathbf{0.206}$ & $\mathbf{0.581}$ & $\mathbf{1.000}$ & $\underline{38.0}$ & $0.841$ & $0.414$ & $1.818$ & $0.289$ & $\mathbf{0.000}$ \\
$10^{-3}$ & $\mathbf{0.206}$ & $\mathbf{0.581}$ & $\mathbf{1.000}$ & $\mathbf{37.4}$ & $0.817$ & $0.419$ & $1.825$ & $0.208$ & $\mathbf{0.000}$ \\
$10^{-2}$ & $\mathbf{0.206}$ & $\mathbf{0.581}$ & $\mathbf{1.000}$ & $\underline{37.6}$ & $\underline{0.782}$ & $0.456$ & $1.877$ & $\underline{0.125}$ & $\mathbf{0.000}$ \\
$10^{-1}$ & $\mathbf{0.206}$ & $\mathbf{0.581}$ & $\mathbf{1.000}$ & $\underline{37.9}$ & $\mathbf{0.693}$ & $0.593$ & $2.013$ & $\mathbf{0.036}$ & $\mathbf{0.000}$ \\
\bottomrule
\end{tabular}%
\end{table}

\subsection{Additional Results on Robustness}
\label{app:additional_results}
Here, we compare \our with ASUN and PCA on how much the canonical representations move upon random small perturbations. We take $500$ random perturbation directions and report the fraction of representations that move more than $0.1$ of the cloud's RMS radius. We take multiple point cloud geometries, some more robust than the others in terms of immunity to big jumps upon random perturbations. We sort the geometries according to their jump percentage and report the fraction of representations that flip significantly for each of PCA, ASUN, and \our in Figure~\ref{fig:geometry-census}.

\begin{figure}[htb!]
    \centering
    \includegraphics[width=0.6\textwidth]{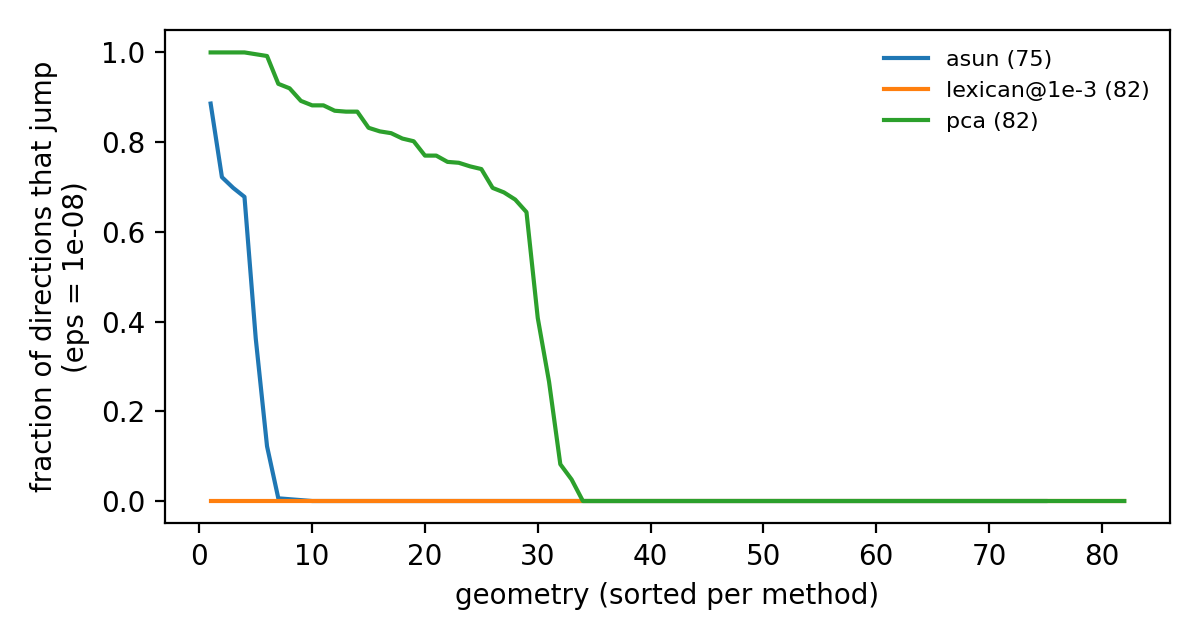}
    \caption{Each point represents the fraction of $500$ random perturbation directions, at $\varepsilon=10^{-8}$, that move the canonical form by more than $0.1$ of the cloud's RMS radius, for each geometry. The geometries are sorted descending on the x axis independently per method. ASUN is drawn over the $75$ geometries it can process; it raises an assertion on the octahedron and the six bipyramids at every tolerance.}
    \label{fig:geometry-census}
\end{figure}

At the tested perturbation magnitude, \our exhibits no
threshold-exceeding jumps across the 82 pilot geometries.
PCA and ASUN exhibit positive observed jump frequencies on
several geometries, with substantial variation across inputs.
These results support lower observed perturbation sensitivity
for \our under this protocol.

\textbf{Machine dependence of PCA.} The PCA values in Tables~\ref{tab:disc-methods} and \ref{tab:path-all} were computed on an Intel Xeon Gold 6336Y CPU with NumPy~2.5.3 (OpenBLAS~0.3.34). They may differ on other machines, since PCA's eigenbasis at repeated eigenvalues is set by floating-point rounding, which depends on the processor's linear-algebra kernels (for example, $5{,}680$ instead of $5{,}301$ discontinuities on the $\mathrm{CH}_4$ stretch on an Intel Xeon 6710E). The conclusions drawn from these tables are unchanged.

\subsection{Perturbation Across Symmetric Polyhedra}
\label{app:symmetric-polyhedra}
In Table~\ref{tab:PCA-ASUN-Lexi-flip}, we consider the effect of adding isotropic Gaussian noise at $\sigma$ from $10^{-2}$ to $10^{-7}$ to a set of polyhedra described in Section~\ref{app:symmetric-polyhedra-data}.

We quantify our results by considering the `flip rate' and the failure to be invariant (the canonicalization is not unique, measured by considering the distance between canonicalized inputs at some tolerance). 

Our algorithm shows $0$ clouds with invariance errors (the `Not invariant' column in the table below) against $18/18$ and $6/17$ for baselines, and the $\mathrm{p}95$ amplification rate of the worse version of \our is $30$ times smaller than the best baseline. ASUN raises an error on one of the clouds (see Appendix~\ref{app:tolerance-ASUN-PCA}), which is why ASUN is compared against one less example. ASUN is unable to return a frame for certain configurations, as it requires three non-perpendicular directions. In this setup, it fails against a regular octahedron with its two axial vertices labeled differently from its four equatorial ones.

\begin{table}[H]
\centering
\caption{Algorithms applied to the $18$ labeled point clouds spanning polyhedra, molecules, and adversarial cases (see Figure~\ref{fig:zoo-gallery}). Flip rate is the fraction of perturbations whose canonical form varied $10\times$ further than the input moved, $\alpha$ is the output movement divided by the input movement (1 means no amplification), and not invariant counts the clouds whose random rotations give different canonical forms, out of those the algorithm canonicalizes. Best results are in bold.}
\label{tab:PCA-ASUN-Lexi-flip}

\begin{tabular}{lcccc}
\toprule
& Flip rate & Median $\alpha$ & p95 $\alpha$ & Not invariant \\
\midrule
PCA & $0.762$ & $1.4{\times}10^{3}$ & $7.5{\times}10^{6}$ & $18/18$ \\
ASUN @ $10^{-2}$ & $0.356$ & $2.14$ & $1.7{\times}10^{6}$ & $6/17$ \\
\our{} @ $10^{-3}$ & $\mathbf{0.196}$ & $\mathbf{1.25}$ & $\mathbf{516}$ & $\mathbf{0/18}$ \\
\our{} @ $10^{-5}$ & $0.404$ & $1.66$ & $5.4{\times}10^{4}$ & $\mathbf{0/18}$ \\
\bottomrule
\end{tabular}%
\end{table}

\subsection{Symmetric Polyhedra Data for Perturbation Experiments}
\label{app:symmetric-polyhedra-data}

We handcraft a set of point clouds for the perturbation experiments. We aim for a computationally manageable sample that would carry a set of challenging configurations, which nonetheless could be representative across domains. It is comprised of:
\begin{itemize}
    \item \textbf{Molecules.} A set of four molecules: NH$_{3}$, CH$_{4}$, H$_{2}$O$_{2}$, and C$_6$H$_6$.
    \item \textbf{Symmetric solids.} Nine polyhedra with structured labels: a tetrahedron and an icosahedron (one marked vertex each); a cube split into two inscribed tetrahedra; a cube with only one species; an octahedron divided into axial and equatorial sites; a hexagonal frustum and antiprism, each split into two rings; a hexagonal bipyramid with marked apices; and a cube-octahedron compound.
    \item \textbf{Unstructured clouds.} Two completely asymmetric configurations: nine isotropic Gaussian points (two species) and twelve random spherical directions (three species).
    \item \textbf{Adversarial cases.} Three configurations designed to defeat specific tie-breaking rules. In a cuboctahedron and a cube-octahedron compound, species are assigned by vertex index to break structural symmetry. The third case is a spherical shell with one point with a different label at one-tenth the outer radius. All three of these were meant to test the limits of \our; the first two have full symmetries that cannot be easily broken by distance, angle or label, while the spherical shell forces a complete tie at the second step of the algorithm (after setting the $z$ axis).
\end{itemize}

\begin{figure}[H]
    \centering
    \includegraphics[width=\textwidth]{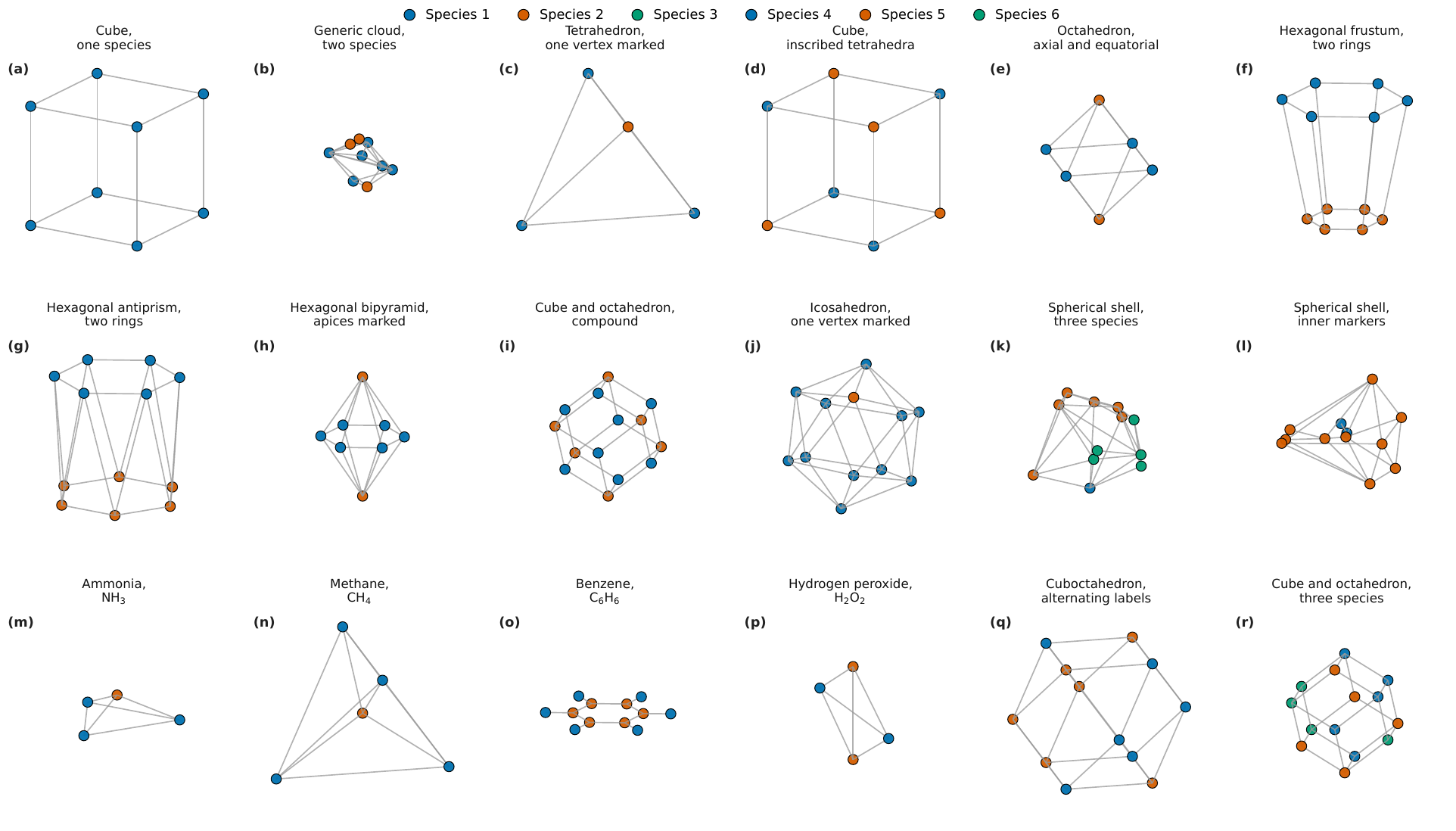}
    \caption{The $18$ labeled point clouds used in the perturbation study. The set represents a mix of four molecules, symmetric polyhedra, unstructured clouds and three adversarial cases. Color indicates species, while edges are drawn only to make
    the shape legible and carry no information beyond the points themselves.
    Each cloud is scaled independently, so relative size across panels is not
    meaningful.}
    \label{fig:zoo-gallery}
\end{figure}

\subsection {Tolerance Threshold Dependence of \our in End-to-end Learning}

In the main experimental section, we have compared \our with baselines across a broad range of tasks testing speed, invariance, robustness and performance in end-to-end learning tasks. However, we had chosen the tolerance threshold of $10^{-3}$ for those tasks. As seen in Tables~\ref{tab:tol-sweep-zoo} and \ref{tab:tol-sweep-qmsym}, having a lower threshold ensures invariance (thus, an accurate canonicalization) and runtime efficiency, but might hurt robustness. 

However, it is unclear how that would impact end-to-end learning performance. Thus in this section, we demonstrate the performance of \our across two thresholds: the chosen threshold of $10^{-3}$ and a lower threshold of $10^{-5}$. We show the performance variation of \our across the end-to-end learning tasks in the following Tables~\ref{tab:qm9_asun_comparison_app} and \ref{tab:results-n512_app}:

\begin{table}[H]
\centering
\setlength{\tabcolsep}{2pt}
\caption{QM9 isotropic polarizability prediction (Transformer). Test MAE and RMSE over three seeds ($\text{mean} \pm \text{STD}$), on randomly rotated, permuted and translated test molecules, as in Table~\ref{tab:qm9_asun_comparison}. Regular is trained on the stored QM9 poses.}
\label{tab:qm9_asun_comparison_app}

\begin{tabular}{@{}lcc@{}}
\toprule
Representation & MAE & RMSE \\
\midrule
Regular & $0.81{\pm}0.07$ & $1.22{\pm}0.09$ \\
ASUN & $0.46{\pm}0.04$ & $0.70{\pm}0.04$ \\
\our$_{@10^{-3}}$ & $0.40{\pm}0.04$ & $0.60{\pm}0.06$ \\
\our$_{@10^{-5}}$ & $\mathbf{0.38{\pm}0.02}$ & $\mathbf{0.59{\pm}0.03}$ \\
\bottomrule
\end{tabular}
\end{table}

\begin{table}[H]
\centering
\caption{Accuracy of classification (\%, mean $\pm$ STD over three seeds). Best results are in bold.}
\label{tab:results-n512_app}

\begin{tabular}{@{}lccccc@{}}
\multicolumn{6}{c}{ModelNet40 ($n=512$)} \\
\cmidrule(lr){1-6}
Model & $z/\SO(3)$ & $\SO(3)/\SO(3)$ & ASUN & \our$_{@10^{-3}}$ & \our$_{@10^{-5}}$ \\
\midrule
SchNet & $27.8\pm0.5$ & $\bm{28.7\pm0.3}$ & $28.6\pm0.7$ & $28.1\pm0.1$ & $28.7\pm1.1$ \\
EGNN & $\bm{34.9\pm0.7}$ & $31.4\pm4.9$ & $32.2\pm0.8$ & $33.4\pm0.1$ & $33.7\pm4.0$ \\
\midrule
GCN & $8.2\pm0.6$ & $15.2\pm0.2$ & $17.1\pm2.5$ & $\bm{28.6\pm1.3}$ & $27.9\pm1.7$ \\
GCNII & $6.9\pm0.8$ & $19.1\pm1.1$ & $21.9\pm1.3$ & $\bm{29.1\pm2.7}$ & $27.8\pm0.4$ \\
SchNet-Pos & $10.5\pm1.1$ & $29.2\pm2.3$ & $30.0\pm0.8$ & $\bm{37.1\pm3.2}$ & $31.8\pm0.4$ \\
EGNN-Pos & $10.3\pm1.4$ & $31.5\pm2.3$ & $31.8\pm1.5$ & $37.0\pm0.5$ & $\bm{37.3\pm2.6}$ \\
\bottomrule
\end{tabular}%
\end{table}

\end{document}